\documentclass{article}

\usepackage[accepted]{icml2018}

\usepackage{extra_style}

\icmltitlerunning{Homeomorphic VAE}

\begin{document}

\twocolumn[
\icmltitle{Explorations in Homeomorphic Variational Auto-Encoding}

\icmlsetsymbol{equal}{*}
\begin{icmlauthorlist}
    \icmlauthor{Luca Falorsi}{equal,to}
    \icmlauthor{Pim de Haan}{equal,to}
    \icmlauthor{Tim R. Davidson}{equal,to}
    \icmlauthor{Nicola De Cao}{to}
    \icmlauthor{Maurice Weiler}{to}
    \\
    \icmlauthor{Patrick Forré}{to}
    \icmlauthor{Taco S. Cohen}{to,to2}
\end{icmlauthorlist}

\icmlaffiliation{to}{University of Amsterdam}
\icmlaffiliation{to2}{Qualcomm AI Research}

\icmlcorrespondingauthor{Luca Falorsi}{luca.falorsi@gmail.com}

\icmlkeywords{Machine Learning, Variational Inference, Bayesian Learning, SO(3), Lie Groups, VAE}

\vskip 0.3in
]
\printAffiliationsAndNotice{\icmlEqualContribution}

\begin{abstract}
The manifold hypothesis states that many kinds of high-dimensional data are concentrated near a low-dimensional manifold.
If the topology of this data manifold is non-trivial, a continuous encoder network cannot embed it in a one-to-one manner without creating holes of low density in the latent space. 
This is at odds with the Gaussian prior assumption typically made in Variational Auto-Encoders (VAEs), because the density of a Gaussian concentrates near a blob-like manifold.

In this paper we investigate the use of manifold-valued latent variables.
Specifically, we focus on the important case of continuously differentiable symmetry groups (Lie groups), such as the group of 3D rotations $\SO3$.
We show how a VAE with $\SO3$-valued latent variables can be constructed, by extending the reparameterization trick to compact connected Lie groups. Our experiments show that choosing manifold-valued latent variables that match the topology of the latent data manifold, is crucial to preserve the topological structure and learn a well-behaved latent space. 
\end{abstract}

\section{Introduction} \label{sec:intro}

Many complex probability distributions can be represented more compactly by introducing latent variables.
Intuitively, the idea is that there is some simple underlying latent structure, which is mapped to the observation space by a potentially complex nonlinear function.
It will come as no surprise then, that most research effort has aimed at using maximally simple priors for the latent variables (e.g. Gaussians), combined with flexible likelihood functions (e.g. based on neural networks).

\begin{figure}[!ht]
\vskip 0.2in
\begin{center}
\centerline{\includegraphics[width=0.75 \columnwidth]{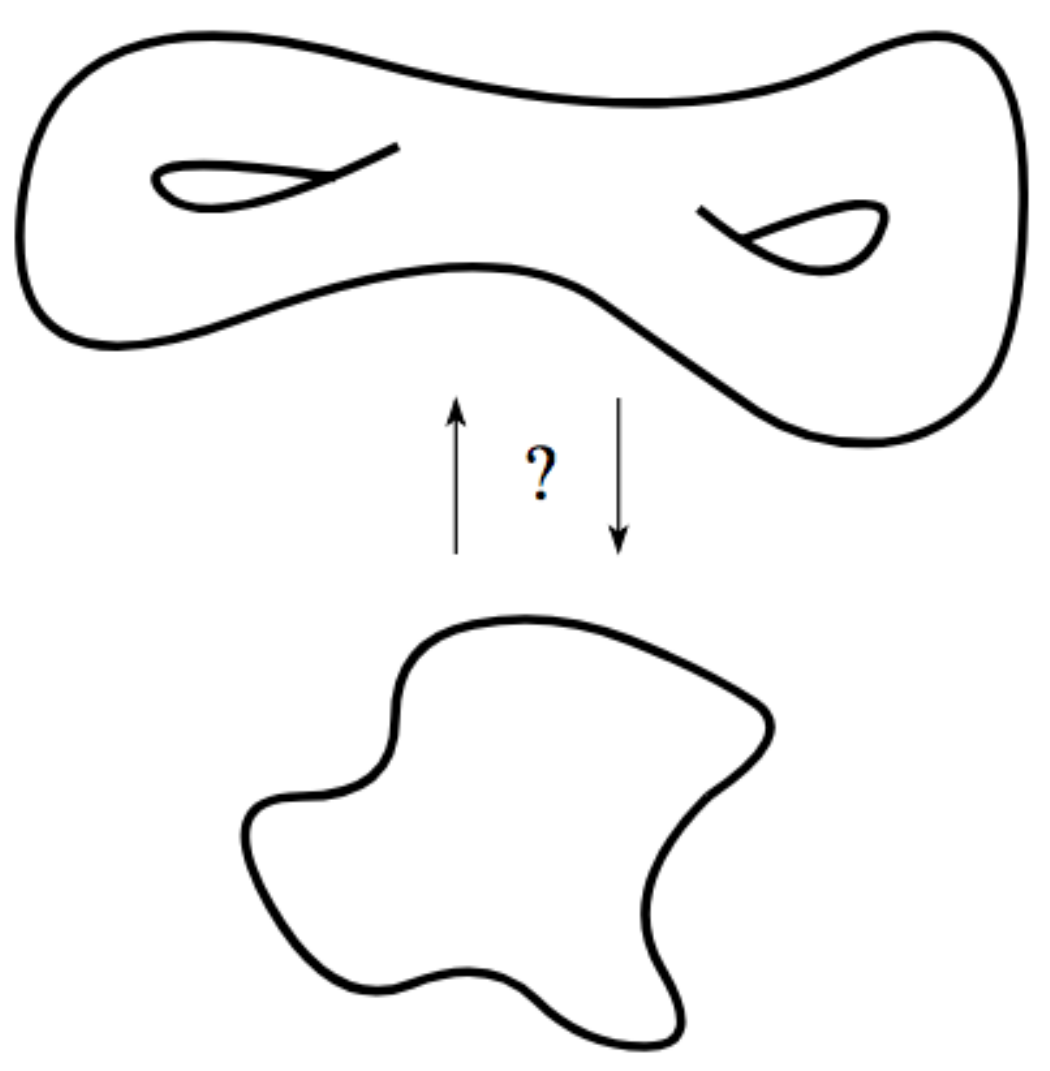}}
\caption{An example of problems that arise in mapping manifolds not diffeomorphic to each other. Notice that in the illustrated example the `holes' in the first manifold, prevent a smooth mapping to the second.}
\label{fig:mani-map}
\end{center}
\vskip -0.2in
\end{figure}

However, it is not hard to see (Fig. \ref{fig:mani-map}) that if the data is concentrated near a low-dimensional manifold with non-trivial topology, there is no continuous and invertible mapping to a blob-like manifold (the region where prior mass is concentrated).
We believe that for purposes of representation learning, the embedding map (encoder) should be homeomorphic (i.e. continuous and invertible, with continuous inverse), which means that although dimensionality reduction and geometrical simplification (flattening) may be possible, the topological structure should be preserved.

Once could encode such a manifold in a higher dimensional flat space with a regular variational auto-encoder (VAE, \citet{KingmaW13-vae, rezende2014stochastic}), rather than learning a homeomorphism. This has two disadvantages. The prior on the flat space will put density outside of the embedding and traversals along the extra dimensions that are normal to the manifold will either leave the decoding invariant, or move out of the data manifold. This is because at each point there will be many more degrees of freedom than the dimensionality of the manifold. 

In this paper we investigate this idea for the special case of Lie groups, which are symmetry groups that are simultaneously differentiable manifolds.
Lie groups include rotations, translations, scaling, and other geometric transformations, which play an important role in many application domains such as robotics and computer vision. More specifically, we show how to construct\footnote{Our implementation is available at \url{https://github.com/pimdh/lie-vae}.} a VAE with latent variables that live on a Lie group, which is done by generalizing the reparameterization trick.

We will describe an approach for reparameterizing densities on $\SO3$, the group of 3D rotations, which can be extended to general compact and connected Lie group VAEs in a straightforward manner. The primary technical difficulty in the construction of this theory is to show that the pushforward measure induced by our reparameterization has a density that is absolutely continuous w.r.t. the Haar measure. Moreover, we show how to construct the encoder such that it can learn a homeomorphic  map from the data manifold to the mean parameter of the posterior. Finally, we propose a decoder that uses the group action to further encourage the latent space to respect the group structure.

We perform experiments on two types of synthetic data: $\SO3$ embedded into a high dimensional space through its group representation, and images of 3D rotations of a single colored cube. We find that a theoretically sound architecture is capable of continuously mapping the data manifold to the latent space. On the other hand, models that do not respect topological structure, and in particular those with a standard Gaussian latent space, show discontinuities when trajectories in the latent space are visualized.
To better study this phenomenon, we introduce a way to measure the continuity of the embedding based on the concept of Lipschitz continuity.
We empirically demonstrate that only  a manifold-valued latent variable with the required topological structure is capable of fully solving the difficult task of the more complicated experiment.

Our main contributions in this work are threefold:
\begin{enumerate}
    \item A reparameterization trick for distributions on the $\SO3$ group of rotations in three dimensions.
    \item An encoder for the mean parameter that learns a homeomorphism between the $\SO3$ manifold embedded in the data and $\SO3$ itself.
    \item A decoder that uses the group action to respect the group structure.
\end{enumerate}

\section{Preliminary Concepts} \label{sec:prelim}
In this section we will first cover a number of preliminary concepts that will be used in the rest of the paper.

\subsection{Variational Auto-Encoders} \label{subsec:vae}

The VAE is a latent variable model, in which $\x$ denotes a set of observed variables, $\z$ stochastic latent variables, and $p(\x,\z)=p(\x|\z)p(\z)$ a parameterized model of the joint distribution called the \emph{generative model}. Given a dataset ${\bf X} = \{ \x_1, \cdots, \x_N\}$, we typically wish to maximize the average marginal log-likelihood $\frac{1}{N}\log p({\bf X}) = \frac{1}{N}\sum_{i=1}^N \log \int p(\x_i, \z_i) d\z$, w.r.t. the parameters. However when the model is parameterized by neural networks, the marginalization of this expression is generally intractable. One solution to overcome this issue is applying variational inference in order to maximize the \emph{Evidence Lower Bound} (ELBO) for each observation:
\begin{align} \label{vae-elbo}
    \log p(\x) &= \log \int p(\x,\z) d\z \nonumber \\
               &\ge \; \E_{q(\z)}[\log p(\x|\z)] - KL(q(\z)||p(\z)),
\end{align}
where the approximate posterior $q(\z)$ belongs to the variational family $\mathcal{Q}$. To make inference scalable an \emph{inference network} $q(\z|\x)$ is introduced that outputs a probability distribution for each data point $\x$, leading to the final objective
\begin{align}
\mathcal{L}(\x ; \theta) =\; \E_{q(\z|\x)}&[\log p(\x|\z)]
- KL(q(\z|\x)||p(\z)),
\end{align}
with $\theta$ representing the parameters of $p$ and $q$.
The ELBO can be efficiently approximated for continuous latent variable $\z$ by Monte Carlo estimates using the \emph{reparameterization trick} of $q(\z | \x)$ \citep{KingmaW13-vae, rezende2014stochastic}.

\subsection{Lie Groups and Lie Algebras}

\paragraph{Lie Group}
A group is a set equipped with a product that follows the four group axioms: the product is closed and associative, there exists an identity element, and every group element has an inverse.
This is closely linked to symmetry transformations that leave some property invariant.
For example, composing two symmetry transformations should still maintain the invariance.
A \emph{Lie group} $G$ has additional structure, as its set is also a smooth manifold.
This means that we can, at least in local regions, describe group elements continuously with parameters.
The number of parameters equals the dimension of the group.
We can see (connected) Lie groups as continuous symmetries where we can continuously traverse between group elements\footnote{We refer the interested reader to \cite{hall2003lie}.}.

\paragraph{Lie Algebra}
The \emph{Lie algebra} $\mathfrak{g}$, of a $N$ dimensional Lie group is its tangent space at the identity, which is a vector space of $N$ dimensions.
We can see the algebra elements as infinitesimal generators, from which all other elements in the group can be created. For matrix Lie groups we can represent vectors $\bm{v}$ in the tangent space as matrices $\mathbf{v}_\times$.

\paragraph{Exponential Map}
The structure of the algebra creates a map from an element of the algebra to a vector field on the group manifold.
This gives rise to the \emph{exponential map} $\exp : \mathfrak{g} \rightarrow G$ which maps an algebra element to the group element at unit length from the identity along the flow of the vector field. The zero vector is thus mapped to the identity. For compact connected Lie groups, such as $\SO3$, the exponential map is surjective.

\begin{figure*}[h!t]
    \centering
    \includegraphics[width=0.75\textwidth]{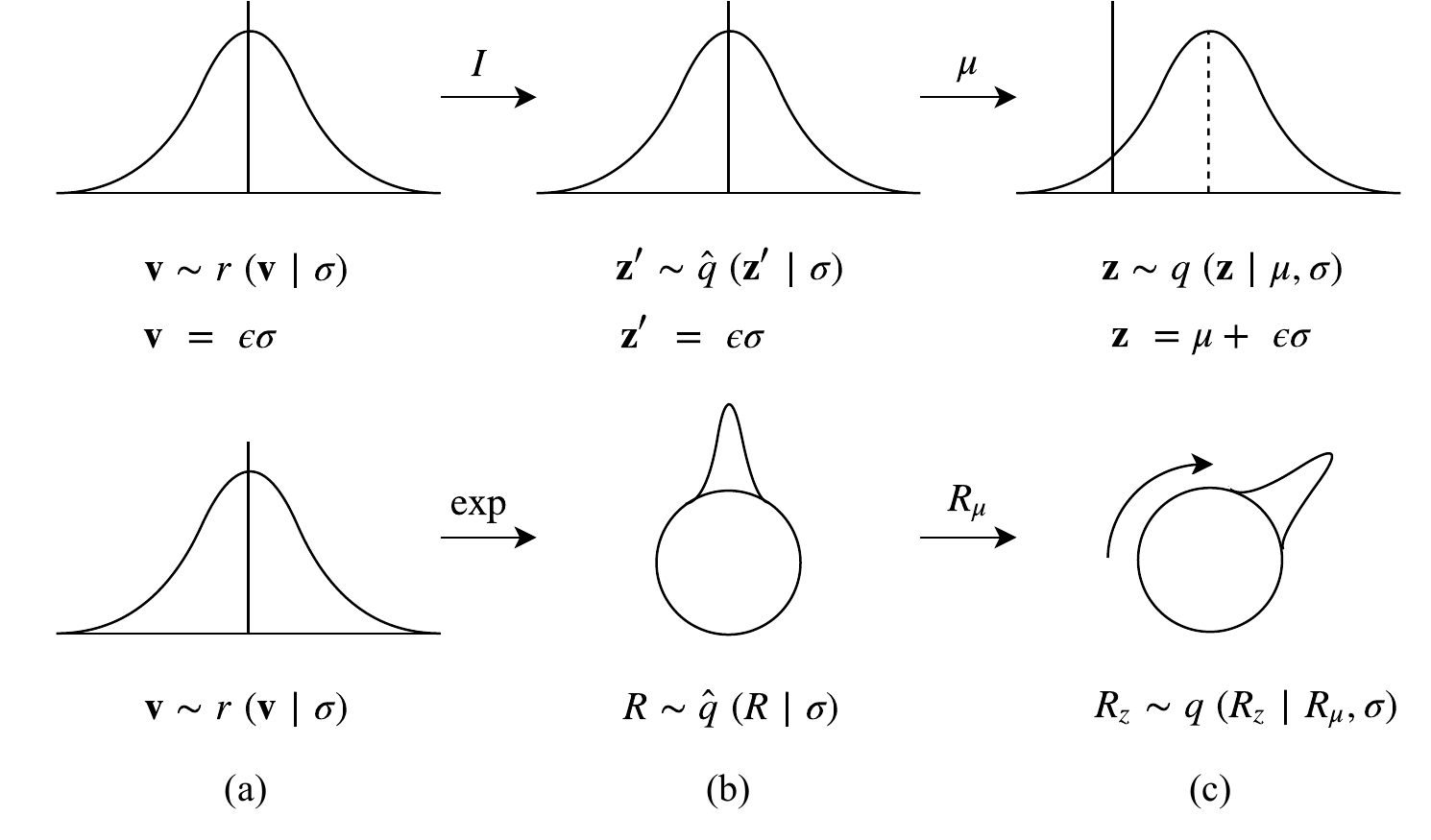}
    \caption{Illustration of our extended reparameterization trick in comparison to the classic reparameterization trick.}
    \label{fig:sample-so3-pipeline}
\end{figure*}

\subsection{The group \SO{3}}

The special orthogonal Lie group of three dimensional rotations $\SO3$ is defined as:
\begin{equation} \label{eq:so3-def}
    \SO3:= \{R \in GL(\mathbb{R}^3) : R^\top R = I \land det(R) = 1 \}    
\end{equation}
where $GL$ is the \emph{general linear group}, which is the set of square invertible matrices under the operation of matrix multiplication. Note that $\SO3$ is not homeomorphic to \Rn, since on \Rn every continuous path can be continuously contracted to a point, while this is not true for $\SO3$. Consider for example a full rotation around a fixed axis.

The elements of Lie algebra $\mathfrak{so}(3)$ of group $\SO3$, are represented by the 3 dimensional vector space of the skew-symmetric $3 \times 3$ matrices. We choose a basis for the algebra:
\begin{align}
   \hspace{-.75em}L_{1,2,3} := \begin{bmatrix} 0 & 0 & 0 \\ 0 & 0 & -1 \\ 0 & 1 & 0 \end{bmatrix}, \begin{bmatrix} 0 & 0 & 1 \\ 0 & 0 & 0 \\ -1 & 0 & 0 \end{bmatrix}, \begin{bmatrix} 0 & -1 & 0 \\ 1 & 0 & 0 \\ 0 & 0 & 0 \end{bmatrix}
\end{align}
This provides a vector space isomorphism between $\mathbb{R}^3$ and $\mathfrak{so(3)}$, written as $[\;\cdot\;]_\times : \mathbb{R}^3 \to  \mathfrak{so(3)}$.

Assuming the decomposition $\mathbf{v}_\times = \theta \mathbf{u}_\times$, s.t. $\theta \in \mathbb{R}_{\ge 0}, \; \|\mathbf{u} \| =1$, the exponential map is given by the Rodrigues rotation formula \citep{rodrigues1840lois}:
\begin{align}
    \exp(\mathbf{v}_\times) = \mathbf{I} + \sin(\theta)\mathbf{u}_\times + 
    (1 - \cos(\theta))\mathbf{u}_\times^2
\end{align}
Since $\SO3$ is a compact and connected Lie group this map is surjective, however it is not injective. 

\section{Reparameterizing $\SO3$} \label{sec:reparam-so3}

In this section we will explain our $\SO3$ reparameterization trick by analogy to the classic version described in \citep{KingmaW13-vae, rezende2014stochastic}. An overview of the different steps and their relation to the classical case are given in Figure \ref{fig:sample-so3-pipeline}.

We sample from a scale-reparameterizable distribution $r(\vv | \sigma)$ on $\mathbb{R}^3$ that is concentrated around the origin.
Due to the isomorphism between $\mathbb{R}^3$ and $\mathfrak{so(3)}$ this can be identified with a sample $\mathbf{v}_\times$ from the Lie algebra $\mathfrak{so(3)}$.
Next we apply the exponential map to obtain a sample $R = \exp(\mathbf{v}_\times) \sim \hat{q}(R | \sigma)$ of the group as visualized in Figure \ref{fig:sample-so3-pipeline} (a) to (b).
Since the distribution $r(\vv | \sigma)$ is concentrated around the origin, the distribution of $\hat{q}(R | \sigma)$ will be concentrated around the group identity.
In order to change the location of the distribution $\hat{q}$, we left multiply $R$ by another element $R_\mu$, see Figure \ref{fig:sample-so3-pipeline} (b) to (c).

To see the connection with the classical case, identify $\mathbb{R}^N$ under addition as a Lie group, with the Lie algebra isomorphic to $\mathbb{R}^N$.
As the group and the algebra are in this case isomorphic, the step of taking the exponential map can be taken as the identity operation such that $r = \hat{q}$.
The multiplication with a group element to change the location corresponds to a translation by $\mu$.

One critical complication is that it is not obvious that the measure we defined above through the exp map has a density function $p : \SO3 \to \mathbb{R}_+$. For this to be the case we need the constructed measure to be absolutely continuous with respect to the Haar measure $\nu$, the natural measure on the Lie group. This is proven by the following theorem.

\begin{thm}
Let $(\mathbb{R}^3,\lambda, \mathcal{B}[\mathbb{R}^3] )$ the real space, provided with the Lebesgue measure on the Borel algebra on $\mathbb{R}^3$. Let $(\SO3,\nu,\mathcal{B}[\SO3])$ the group of 3 dimensional rotations, provided with the normalized Haar measure $\nu$ on the Borel $\sigma$-algebra on $\SO3$. Consider then the probability measure $\mu:\mathcal{B}[\mathbb{R}^3]\to [0,1]$ absolutely continuous w.r.t $\lambda$, with density $r$. Consider the exponential map $\exp: \mathbb{R}^3 \to \SO3$ that is differentiable, thus continuous, thus measurable.
Let then $\exp_*(\mu)$ be the pushforward of $\mu$ by the $\exp$ function. then $\exp_*(\mu)$ is absolutely continuous with respect of the Haar measure $\nu$ ($\exp_*(\mu) \ll \nu$). 
\end{thm}
\begin{proof}
    See Appendix \ref{appendix:pushforward-so3}
\end{proof}
As further derived in Appendix \ref{appendix:pushforward-so3} this implies the pushforward measure on $\SO3$ to be absolutely continuous w.r.t. to the Haar measure where the density 
\begin{equation}
\hat q(R|\sigma) = \sum_{k \in \mathbb{Z}} r\left(\frac{\log(R)}{\theta(R)}(\theta(R) + 2k\pi) \bigg| \sigma\right)\frac{(\theta(R) + 2k\pi)^2}{3 - \text{tr}({R})},
\end{equation}
is defined almost everywhere.
Here $R\in \SO3$ and
\begin{equation}
\theta({R}) = \|\log(R)\| = \cos^{-1}\left(\frac{\text{tr}(R)-1}{2}\right)
\end{equation}
Further, $\log(\cdot)$ is defined as a principal branch and maps back the group element to the unique Lie algebra element next to the origin.
Notice that even if the density is singular at $R=I$, it still integrates to 1.
After rotating $R$ by left multiplying with another $\SO3$ element $R_\mu$, we obtain the final sample:
\begin{equation}
    R_z \sim q(R_z | R_\mu, \sigma) = \hat{q}(R_\mu^\top R_z | \sigma),
\end{equation}
where the second step is valid because of the left invariance of the Haar measure.

\paragraph{Kullback-Leibler Divergence}
The KL divergence, or \emph{relative entropy} can be decomposed into the \emph{entropy} and the \emph{cross-entropy} terms, $KL(q || p) = \mathbb{H}(q, p) - \mathbb{H}(q)$. Since the Haar measure is invariant to left multiplication, we can compute the entropy of the distribution $\hat{q}$ instead of $q$. As we have the expression of the density, the entropy can be computed using Monte Carlo samples:
\begin{align}
\mathbb{H}(q) &= \mathbb{H}(\hat{q}) \approx -\frac{1}{N}\sum_{i=1}^N \log \hat{q}(R_i|\sigma), \quad R_i \sim \hat{q}(R_i|\sigma) \nonumber
            \\
            &= -\frac{1}{N}\sum_{i=1}^N \log  \hat{q}(\exp(\vv_i)|\sigma) \nonumber
            \\
            &= -\frac{1}{N}\sum_{i=1}^N \log \sum_{k \in \mathbb{Z}} r\Bigl(\frac{\vv_i}{\|\vv_i\|} (\|\vv_i\|+ 2k\pi)|\sigma\Bigl) \cdot \nonumber
            \\
            & \hspace{5em} \frac{(\|\vv_i\|+ 2k\pi)^2 }{2 - 2\cos(\|\vv_i\|)},
\quad \vv_i \sim r(\vv_i|\sigma)
\end{align}
Notice that the last expression only depends on the samples taken on the Lie algebra. We found that one sample suffices when mini-batches are used. In general the cross-entropy term can be similarly approximated by MC estimates. However, in the special but important case of a uniform prior, $p$, the cross-entropy reduces to: $\mathbb{H}(p, q) = - \log \left(\frac{1}{8\pi^2} \right) $.

\section{Encoder and Decoder networks}

Having defined the reparameterizable density $q(R_z | R_\mu, \sigma)$, we need to design encoder networks which map elements from the input space $\mathcal{X}$ to the reparameterization parameters $R_\mu$, $\sigma$ and decoder networks which map group elements to the output prediction.

\subsection{Homeomorphic Encoder} \label{sec:encoder}

We split the encoder network in two parts $\text{enc}^\mu$ and $\text{enc}^\sigma$, which predict reparameterization parameters $R_\mu$ and $\sigma$ respectively.
Since $\sigma$ are parameters of a distribution in $\mathbb{R}^3$, the corresponding network $\text{enc}^\sigma$ does not pose any problems and can be chosen similarly as in classical VAEs.
However, special attention needs to be paid to designing $\text{enc}^\mu$ which predicts a group element $R_\mu \in \SO3$.

We consider the data as lying in a lower dimensional manifold $\mathcal{M}$, embedded in the input space $\mathcal{X}$. In our particular problem the manifold $\mathcal{M}$ is assumed to be generated by $\SO3$, acting on a canonical object and a subsequent projection into ambient space (e.g. pixel space) $\mathcal{X}$ which, for simplicity we assume to be injective.
This means that we can make the simplifying assumption that $\SO3$ can be recovered from its image in $\mathcal{X}$, i.e. that the map $\SO3 \to \mathcal{M}$ is a homeomorphism. The encoder is now meant to learn the inverse map, i.e. to learn a map from $\mathcal{X}$ to $\SO3$, which when restricted to $\mathcal{M}$ is a homeomorphism and thus preserves the topological structure of $\SO3$.

In general there is no standard way to define and parameterize the class of functions which are guaranteed to have these properties by design via a neural network.
Instead we will give a general way to build $\text{enc}^\mu$ {\it capable} of learning such a mapping.
We divide the encoder network in two functions: $\text{enc}^{\mu}= \pi\circ f$, where $f : \mathcal{X} \to \mathcal{Y}$, for some space $\mathcal{Y}$, is parametrized by a neural network, and $\pi : \mathcal{Y} \to \SO3$ is a fixed surjective  function.
Not any space $\mathcal{Y}$ or function $\pi$ is suited: since neural networks can only model continuous functions, a necessary condition on $\mathcal{Y}$ for $\text{enc}^\mu$ to be able to learn to be a homeomorphism (when its domain $\mathcal{X}$ is restricted to $\mathcal{M}$), is that there exists an embedding $i:\SO3\to \mathcal{Y}$. Then by definition a function $\pi$ exist, such that  $\pi|_{i(\SO3)}=i^{-1}|_{i(\SO3)}$ is a homeomorphism. Any extension of $\pi|_{i(\SO3)}$ to $\mathcal{Y}$ is a suitable candidate.
Moreover, if we choose $\mathcal{X}=\mathbb{R}^n$ and $\mathcal{Y}=\mathbb{R}^m$ for some $n, m$, then some continuous $f : \mathcal{X} \to \mathcal{Y}$ exists (which we can approximate with neural networks) such that an appropriate $\text{enc}^\mu = \pi \circ f$ exist. Several choices for $\mathcal{Y}$ and $\pi$ are proposed in Appendix \ref{ap:mean-param} and investigated in the experimental section \ref{sec:experiments}.

\subsection{Group Action Decoder} \label{sec:group-act-deco}

Our decoder $p(\x | R_z)$ must be capable to map a group element and optionally additional latent structure back to the original high dimensional input space.
When the factor of variation in the input is the pose of an object, and we learn a latent variable $R_z \in \SO3$, we desire that a transformation $R\in \SO3$ applied to a latent object representation $\z$ results in a corresponding transformation of the pose of the decoded object. The task of the decoder is thus to learn a three dimensional representation of the object, to rotate it according to the latent variable and finally to project it back to the two dimensional frame of the input image. A naive approach could be to simply provide the 9 elements of the rotation matrix to a neural network. However, although it may learn to reconstruct the input well, it provides no guarantee that the latent space accurately reflects the pose variations of the object. Therefore, we like to make the method more explicit.

Hypothetically, one could learn a vector-valued signal on the sphere, $f : S^2 \rightarrow \mathbb{R}^N$, to represent the three dimensional object in the input, as 3D shapes can be well represented by its information projected to a sphere \citep{s.2018spherical}. The decoder can rotate this signal with the latent variable, before projecting it back to pixel space. A major downside of this approach is that parameterizing and projecting a function on the sphere is highly non-trivial.

\begin{figure} 
    \centering
    \includegraphics[width=0.4\textwidth]{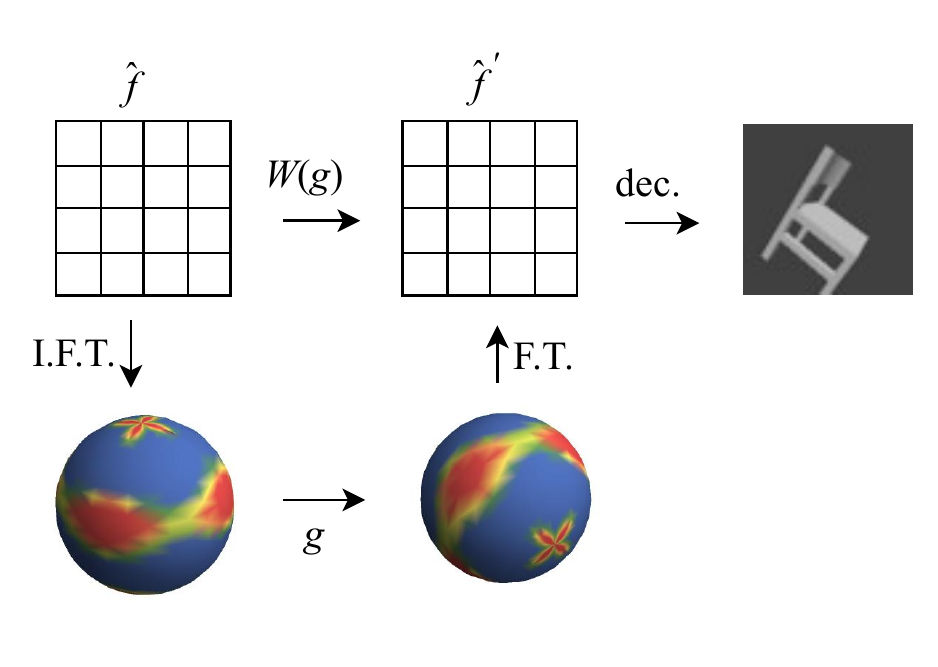}
    \caption{The encoder infers the $R \in \SO3$ and Fourier modes $\hat{f}$ if working with multiple objects, otherwise $\hat{f}$ is a parameter. Shown is the commutative diagram between taking the Inverse Fourier Transform, rotating the result and taking the Fourier Transform, and acting with the group representation $W$ on $\hat{f}$. The decoder maps the transformed $\hat{f}'$ to pixels.}
    \label{fig:action-decoder}
\end{figure}

Alternatively, we propose a method based on the representation theory of groups \citep{hall2003lie}. Rather than learning the function $f$, we directly learn its (band-limited) Fourier modes, which form a simple vector space. It can be shown \citep{chirikjian2000engineering} that rotations of a signal on the sphere correspond to a linear transformation of the Fourier modes. The transformed Fourier modes are subsequently fed through an image generative network, and the linear transformation is the Wigner-D-matrix, which is a function of the $\SO3$ element. Technically, the Wigner-D-matrices form representations of the group. This means that as mapping to the linear transformations is a homomorphism, it preserves the group structure: $D(g)D(g') = D(gg')$, for $g, g' \in \SO3$ and $D(g)$ a Wigner-D-matrix. This method encourages the latent space to represent the actual pose of the input, while only requiring the construction of the $W$ matrices and performing a linear transformation. We refer to Figure \ref{fig:action-decoder} and Appendix \ref{appendix:group-action} for details.

An additional advantage is that this decoder allows for disentangling of content and pose, as it is forced to encode the pose in a meaningful way. The Fourier modes are in that case also generated by the encoder. We leave this for future work.

\begin{figure*}[!ht]
\centering
\subfigure[Ground truth]{\includegraphics[height=0.15\textwidth]{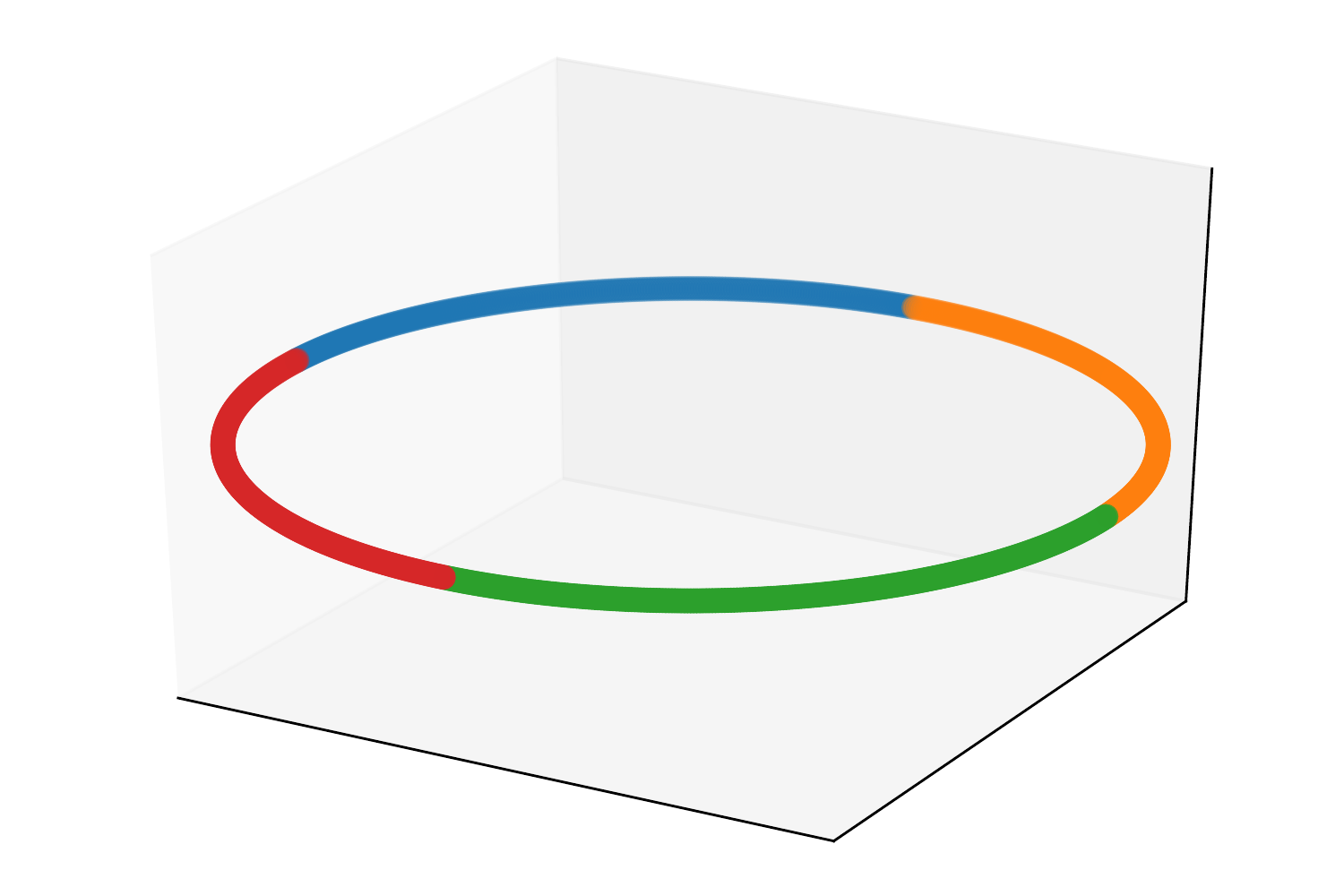}}
\subfigure[$\SO3$-action-s2s2 ]{\includegraphics[height=0.15\textwidth]{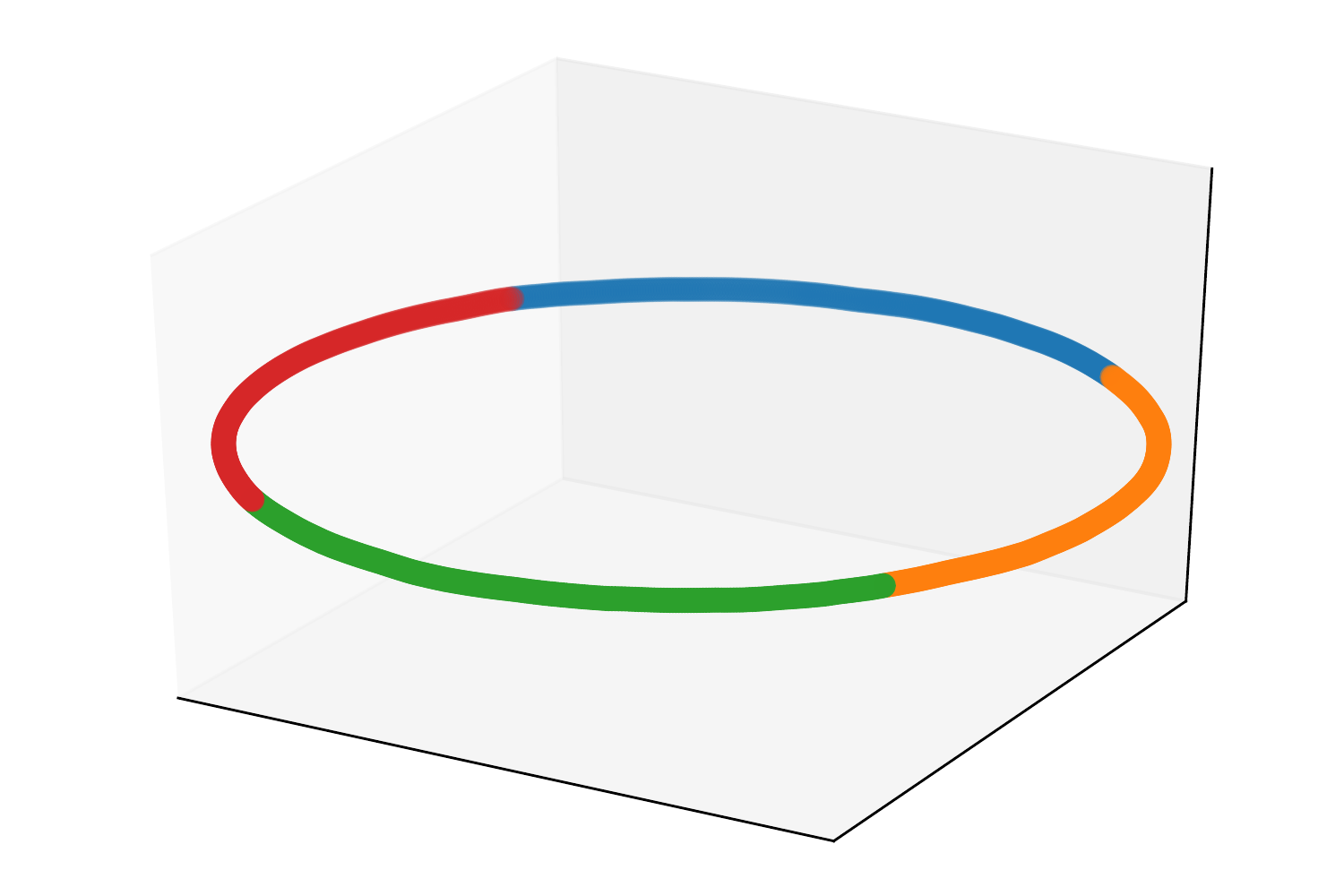}}
\subfigure[$\SO3$-action-alg ]{\includegraphics[height=0.15\textwidth]{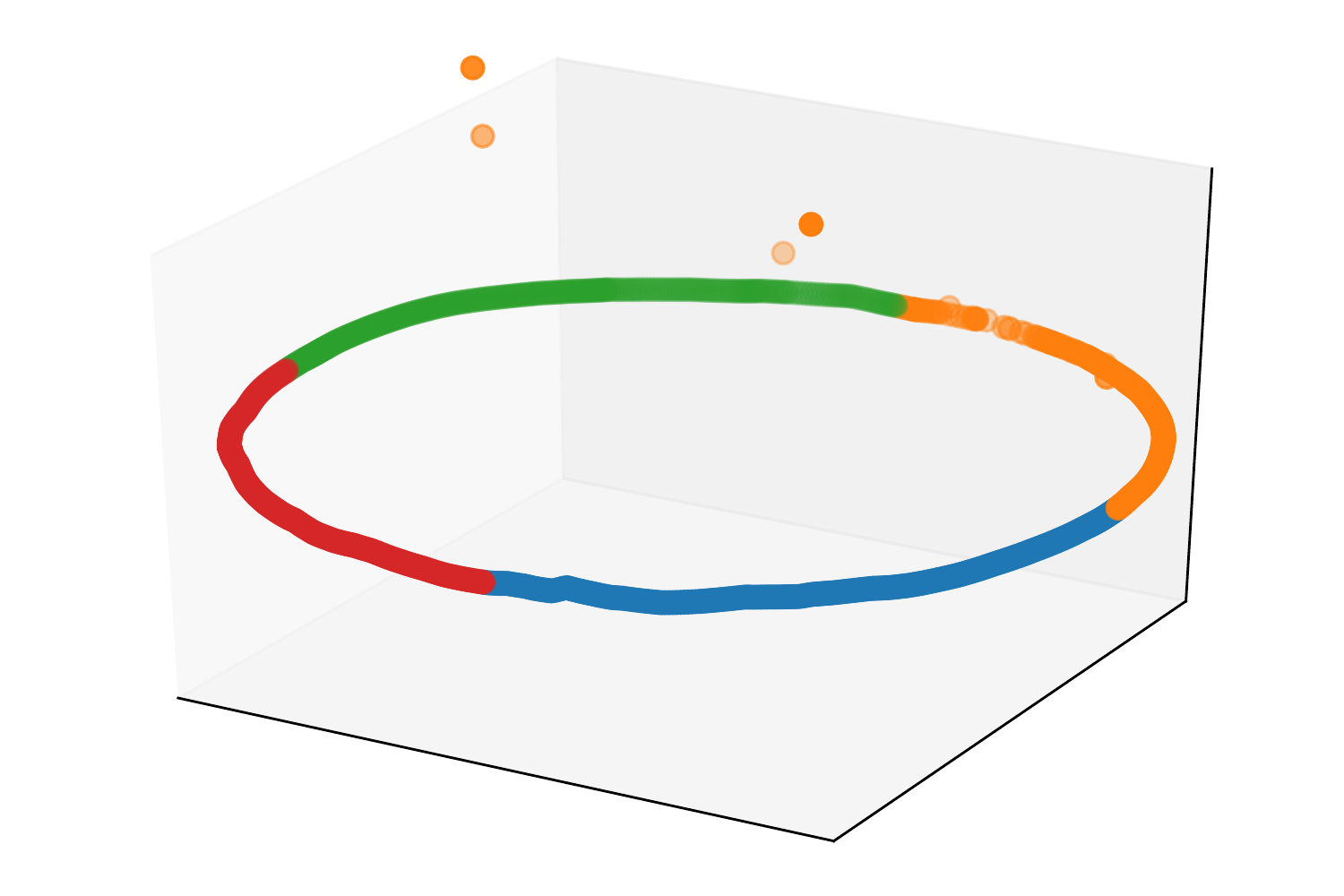}}
\subfigure[$\SO3$-action-q ]{\includegraphics[height=0.15\textwidth]{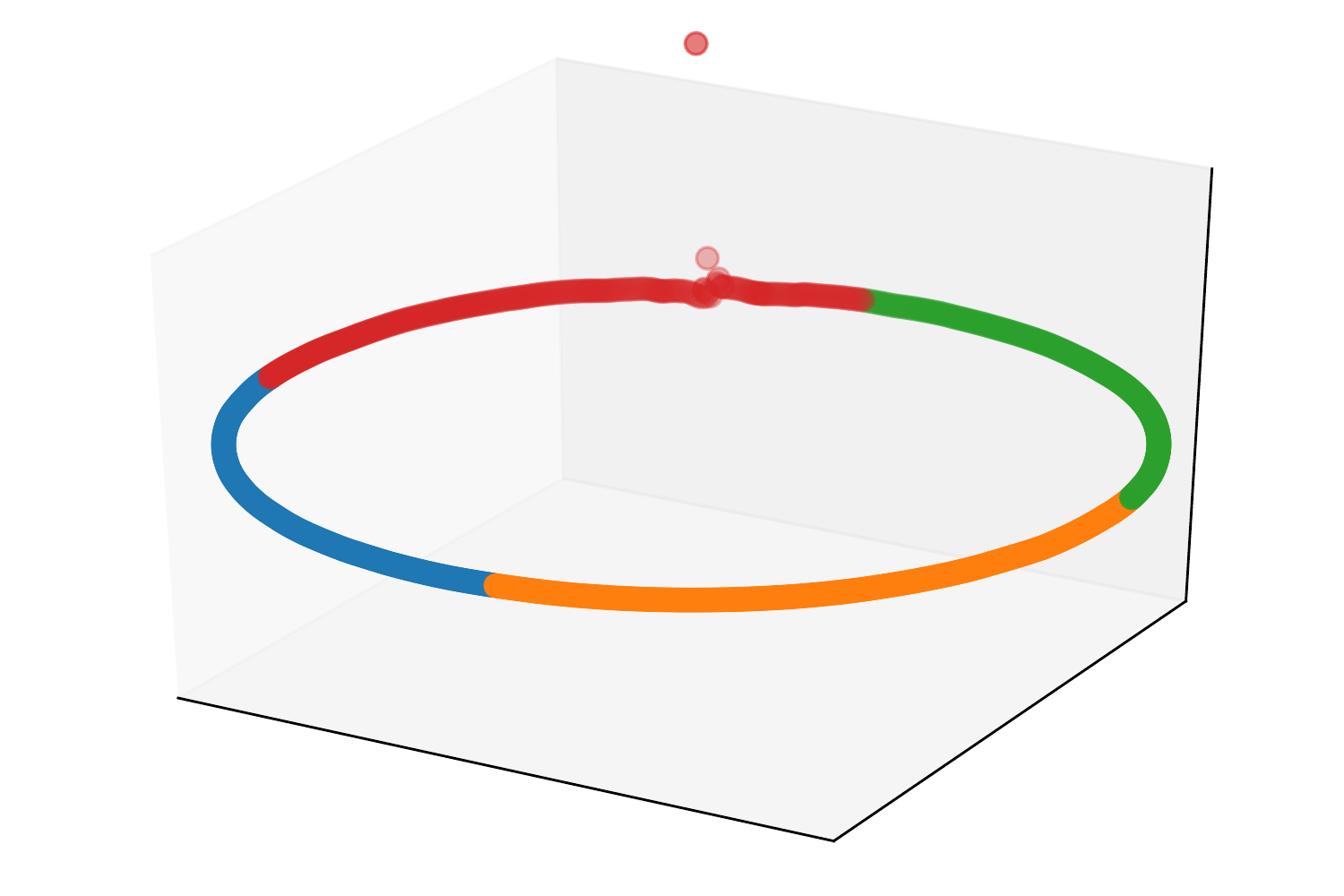}}
\subfigure[$\SO3$-action-s2s1]{\includegraphics[height=0.15\textwidth]{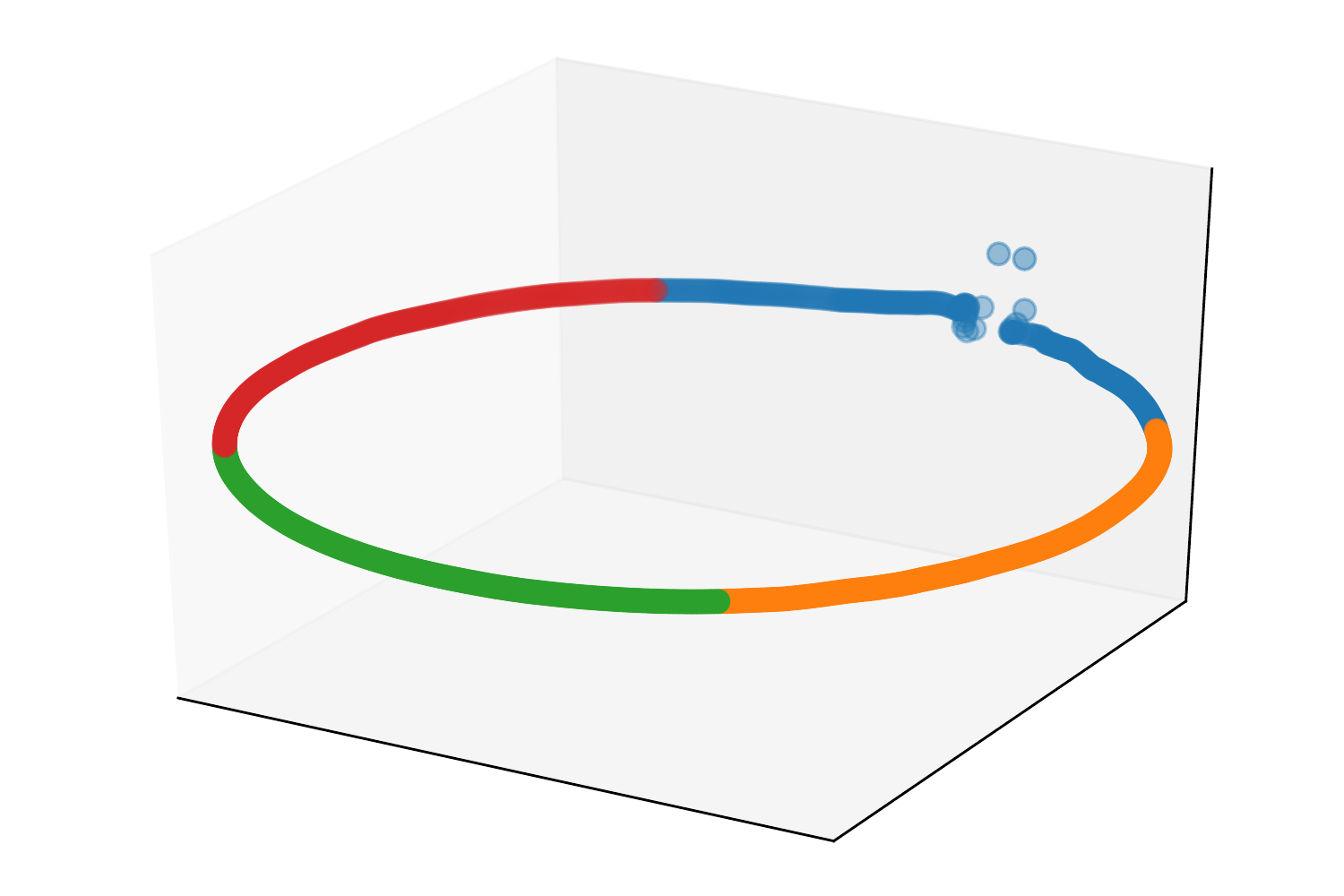}}
\subfigure[$\mathcal{N}$-action-3-dim]{\includegraphics[height=0.15\textwidth]{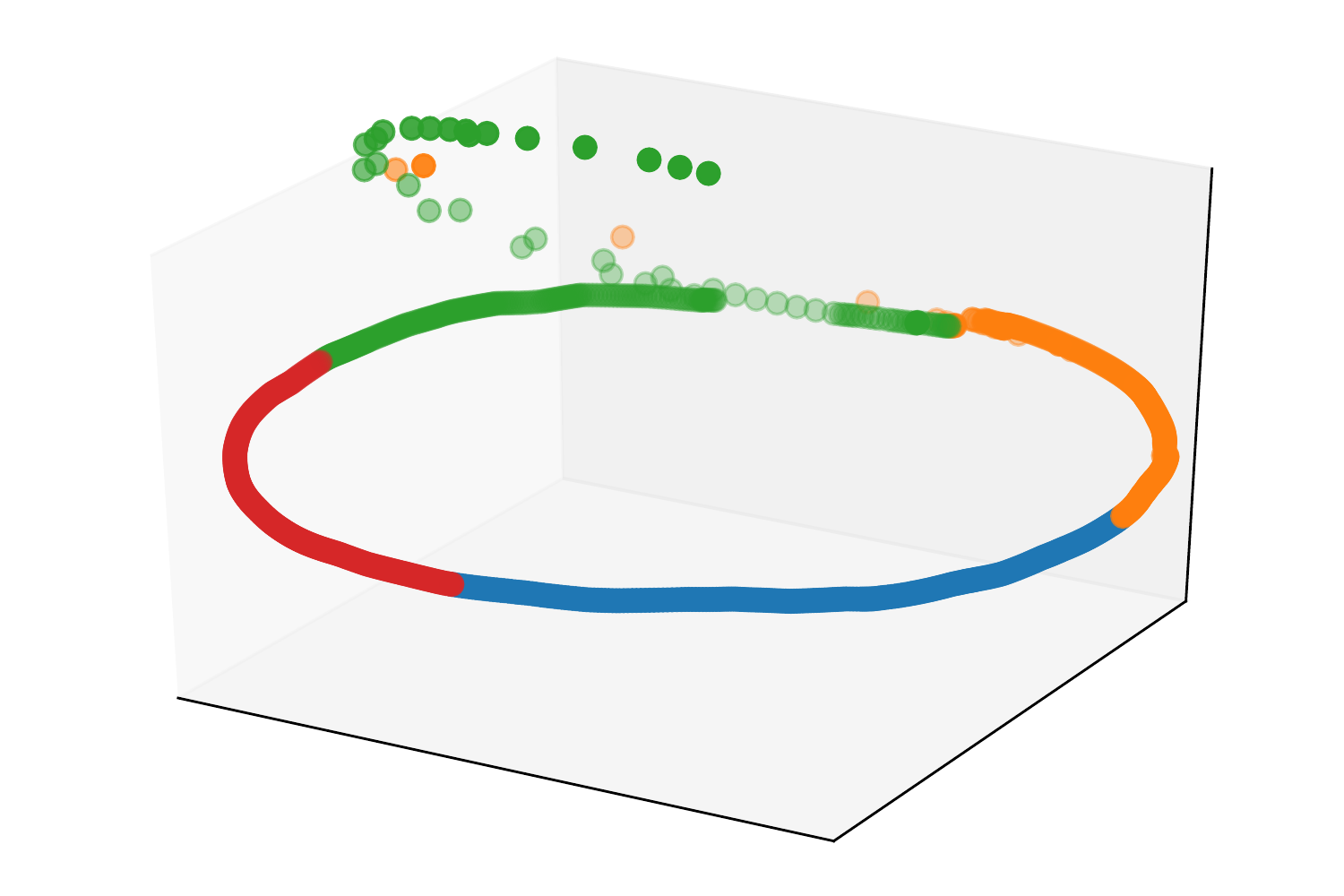}}
\subfigure[$\mathcal{S}$-action-3-dim]{\includegraphics[height=0.15\textwidth]{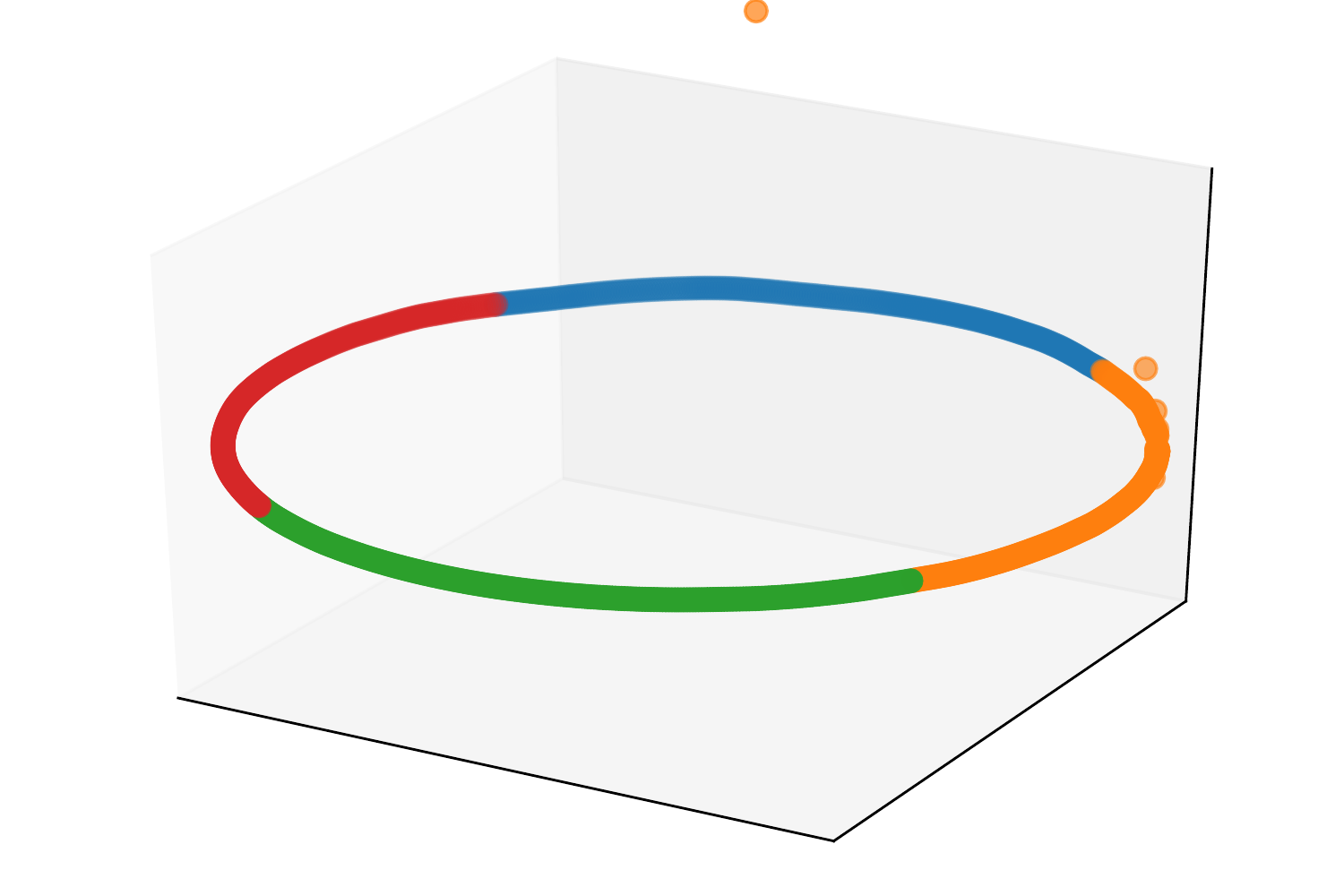}}
\caption{The latent encoding of a $S^1$ trajectory in the Toy data set. The $\SO3$ elements are mapped to $\mathbb{R}^9$ by taking the rotation matrix elements and are subsequently mapped to 3D by Principal Component Analysis.}
\label{fig:toy-latent}
\end{figure*}

\section{Related Work} \label{sec:related-work}

As VAEs utilize VI to recover some distribution on a latent manifold responsible for generating the observed data, the majority of extensions is focused on increasing the flexibility of the prior and approximate posterior. Although the majority of approaches make use of a normal Gaussian prior, recently there has been a surge to provide additional options to offset some of this distribution's perceived limitations. \citet{vamp-prior} propose to directly tie the prior to the approximate posterior and learn it as a mixture over approximate posteriors. \citet{stick} introduce a non-parametric prior applying a truncated stick-breaking method. Research to support discrete latent variables was done in \citet{Jang2016CategoricalRW-gumbel, maddison2016concrete}, while in \citet{naesseth2017reparameterization, figurnov2018implicit} recently novel techniques were introduced to reparameterize a suite of continuous distributions. In \citep{s-vae}, the reparameterization technique of \citet{naesseth2017reparameterization} is extended to explore the properties of the hyperspherical von Mises-Fisher distribution to better capture intrinsically hyperspherical data. This is done in the context of avoiding \emph{manifold mismatches}, and as such is closely related to the motivation of this work. 

The predominant procedure to generate a more complex approximate posterior is through \emph{normalizing flows} \citep{normalizing-flows}, in which a class of invertible transformations is applied sequentially to a reparameterizable density. This general idea has later been extended \citet{kingma2016improved, berg2018sylvester}, to improve flexibility even further. As this framework does not hold any specific distributional requirements on the prior besides being reparameterizable, it would be interesting to investigate possible applications to $\SO3$ in future work.

The problem of defining distributions on homogeneous spaces, including Lie groups, was investigated in \cite{chirikjian2000engineering, chirikjian2010information, Chirikjian2016-ch, Chirikjian2012-gr}.
\citet{cohen2015harmonic} devised harmonic exponential families which are a powerful family of distributions defined on homogeneous spaces.
These works did not concentrate on making the distributions reparameterizable.

Rendering complex scenes from multiple poses has been explored in \citep{eslami1204}. However, this work assumes access to ground truth poses and does not do unsupervised pose learning as in the presented framework.

The idea of incorporating prior knowledge on mathematical structures into machine learning models has proven fruitful in many works.
\citet{s.2018spherical} adapt convolutional networks to operate on spherical and $\SO3$ valued data.
Equivariant networks, investigated in \citet{cohen2016group,cohen2016steerable,Worrall2017-HNET,Weiler2018-STEERABLE,Cohen2018-ex}
reduce the complexity of a learning task by taking a quotient over group orbits which explain a subset of dimensions of the data manifold.

\section{Experiments} \label{sec:experiments}

\begin{figure*}[!ht]
\centering
\subfigure[10 dimensional Normal latent space with MLP decoder]{\includegraphics[width=\textwidth]{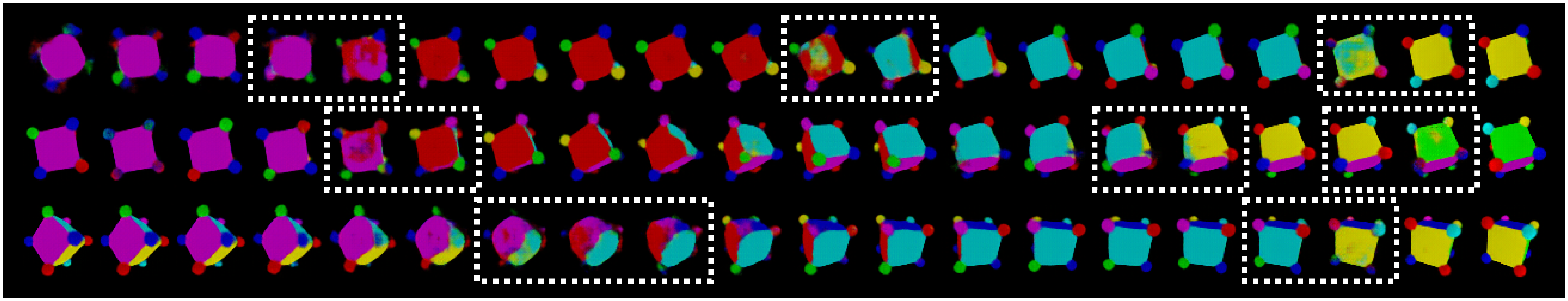}}
\subfigure[$\SO{3}$ latent space with S2S2 mean and action decoder]{\includegraphics[width=\textwidth]{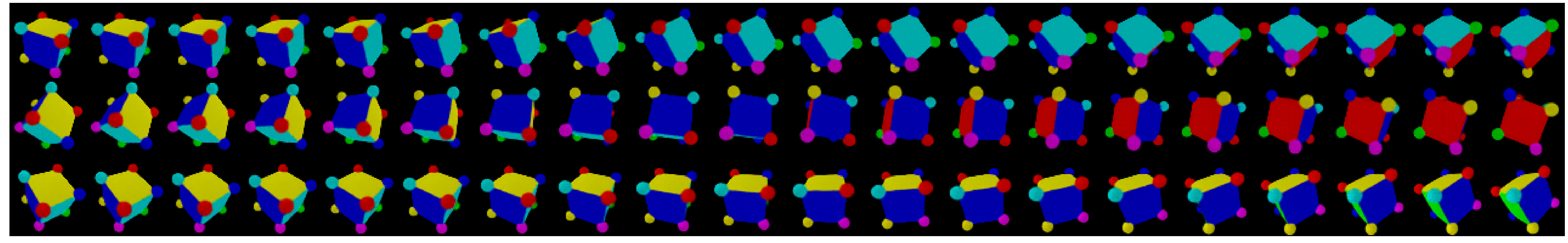}}
\caption{Three interpolations of two models. Discontinuities in the reconstructions of the Normal model are outlines by a dashed line.}
\label{fig:discontinuties}
\end{figure*}

We perform two experiments to investigate the importance of using a homeomorphic parameterization of the VAE in recovering the original underlying $\SO3$ manifold. In both experiments we explore three main axes of comparison: (1) manifold topology, (2) decoder architecture, and (3) specifically for the $\SO3$ models we compare different mean parameterizations as discussed in section \ref{sec:encoder}. For each model we compute a tight bound on the negative log likelihood (NLL) through importance sampling following \citet{burda2015importance}.

For manifold topology we examine VAEs with the Gaussian parameterization (\Nv-VAE), the hyperspherical parameterization of \citet{s-vae} (\Sv-VAE), and the $\SO3$ latent variable discussed above.
The two decoder variants are a simple MLP versus the group action decoder described in section \ref{sec:group-act-deco}. Lastly we explore mean parameterizations through unit Quaternions (q), the Lie algebra (alg), $\mathcal{S}^2 \times \mathcal{S}^1$ (s2s1), and $\mathcal{S}^2 \times \mathcal{S}^2$ (s2s2). These parameterizations are chosen to be either valid ($\mathcal{S}^2 \times \mathcal{S}^2$) or invalid (q, alg, $\mathcal{S}^2 \times \mathcal{S}^1$) for the purpose of investigating the soundness of our theoretical considerations and to compare their behaviour. Details and derivations on the properties of these different parameterizations can be found in Appendix \ref{ap:mean-param}.  

\subsection{Toy experiment} \label{sec:toy}
The simplest way of creating a non-linear embedding of $\SO3$ in a high dimensional space is through the representation as discussed in Section \ref{sec:group-act-deco}. 
The data is created in the following way: a fixed representation $W : \SO3 \to \mathbb{R}^{N \times N}$ is chosen, in this experiment this is three copies of the direct sum of the Wigner-D-matrices up to order 3, making the embedding space $\mathbb{R}^{64}$. Subsequently a single element $v$ of $\mathbb{R}^{64}$ is generated. The data set now consists of the vectors $W(R)v$, where $R \in \SO3$ sampled uniformly. Since the representation is faithful ($W$ is injective) the data points lie on a 3 dimensional submanifold of $\mathbb{R}^{64}$ homeomorphic to $SO(3)$.

In order to verify the ability of the models to correctly learn to encode from the embedded manifold to the manifold itself, we learn various variational and non-variational auto-encoders on this data set. The encoder is a 3 layer MLP, and for the decoder we use the group action decoder of Section \ref{sec:group-act-deco}. The same representation $W$ is used as in the data generation, but we learn $v$. In addition to the $\SO3$ models, we use a 3 dimensional normal, which we map to $\SO3$ using the ZYZ-Euler angles, and a $S^3$ von Mises-Fisher, which we map to $\SO3$ by identifying $S^3$ as the unit quaternions.

\paragraph{Results}
The quantitative results are shown in Table \ref{tab:toy}. We observe that the choice for the mean parametrization significantly impacts the ability of the model to correctly learn the manifold. The $\mathcal{S}^2 \times \mathcal{S}^2$ method strongly outperforms the competing methods in the non-variational Auto-Encoder achieving near-perfect reconstructions. Additionally, the metric indicating the continuity of the encoder, which we define in Appendix \ref{appendix:continuity}, shows it is the only model that does not have discontinuities in the latent space. These results are in line with our theoretical prediction outlined in Section \ref{sec:encoder} and Appendix \ref{ap:mean-param}.

The qualitative results in Figure \ref{fig:toy-latent} and Figures \ref{fig:toy-reconstruction}, \ref{fig:toy-discontinuity} in Appendix \ref{ap:additional-figures} tell a similar story. These plots are created by taking a $S^1$ subgroup of $SO(3)$ and making a $S^1$ submanifold in the data space using the same process with which the data was generated. This embedded trajectory is then encoded and reconstructed. The trajectory is divided in four equally sized partitions, each shown in a different color. We clearly see that only the $\mathcal{S}^2 \times \mathcal{S}^2$ method is able to learn a continuous latent space.

Moreover, the worst performing models are the 3 dimensional $\mathcal{N}$ and $\SO3$ algebra mean models. Interestingly, these share that at one intermediate point $\SO3$ is represented by $\mathbb{R}^3$. This indicates that using flat space to represent a non-trivial manifold results in a poorly structured latent space and worse reconstruction performance and Log Likelihoods.

\begin{table}[!tp]
\centering
\begin{tabular}{lccccc}
\toprule
\multirow{2}{*}{Algorithm} & \multicolumn{3}{c}{VAE} & \multicolumn{2}{c}{AE} \\
& NLL & recon & KL & recon & disc. \\
\midrule
$\SO3$-q      & 10.9 & 2.32 & \textbf{9.16} & 0.29 & .992 \\
$\SO3$-alg    & 13.4 & 6.24 & 9.36 & 4.02 & 1. \\
$\SO3$-s2s1   & 11.0 & 2.12 & 9.41 & 0.29 & 1. \\
$\SO3$-s2s2   & \textbf{10.7} & 1.81 & 9.21 & \textbf{0.01} & 0. \\
$\mathcal{N}$-3-dim      & 18.9 & 9.91 & 10.3 & 14.7 & 1. \\
$\mathcal{S}$-3-dim  & 13.6 & \textbf{1.79} & 11.8 & 0.27 & 1. \\
\bottomrule
\end{tabular}
\caption{Summary of results for the toy experiment for Variational Auto-Encoder (VAE) and Auto-Encoder (AE), including the discontinuity metric (disc.).}
\label{tab:toy}
\end{table}

\subsection{Sphere-Cube} \label{subsec:exp-single-obj}
For this experiment we learn auto-encoders on renderings of a cube. The cube is made highly asymmetrical through the colors of the faces and the colored spheres at the vertices. This should make it easier for the encoder to detect the orientation. This \emph{sphere-cube} is then rotated by applying uniformly sampled group elements from $\SO3$, to create a training set of 1M images. Ideally the model learns to correctly represent these encodings in the latent space.

The encoder consists of 5 convolutional layers, followed by one of the mean encoders and reparameterization methods. The decoder uses either the group action or a 3 layer MLP, both followed by a 5 deconvolutional layers. In order to balance reconstruction and the KL divergence in a controlled manner, we follow \citet{burgess2018understanding} and replace the negative KL term in the original VAE loss with a squared difference of the computed KL value and a target value. We found that a target value of 7 early in training to 15 at the end of the training gave good results. This allows the model to first organize the space and later become more certain of its predictions. We found that two additional regularizing loss terms were needed to correctly learn the latent space. Details can be found in Appendix \ref{appendix:regularizers}.

\paragraph{Results}
Quantitative results comparing the best performing $\SO3$ parameterization to \Nv-VAEs of diff dimensionality are shown in Table \ref{tab:vae-2}. Although the higher dimensional \Nv-VAEs are able to achieve competitive metrics compared to the best $\SO3$ model, they only learn to embed $\SO3$ in a high dimensional space in an unstructured fashion. As can be seen in in \ref{fig:discontinuties}, the $\SO3$ latent space with $\mathcal{S}^2 \times \mathcal{S}^2$ mean parameterization learns a nearly perfect encoding, while the 10 dimensional Normal learns disconnected patches of the data manifold.\footnote{Animated interpolations can be found at \url{https://sites.google.com/view/lie-vae}.}

It can be seen in Table \ref{tab:vae} that the results from the Toy experiment extend to this more complicated task. We observe that only the continuous encoding, $\mathcal{S}^2 \times \mathcal{S}^2$, achieves low log likelihood and reconstruction losses compared to the other mean parameterizations.

Lastly, we observe that the group action decoder yields significantly higher performance than the MLP decoder. This is in line with the hypotheses that using the group action encourages structure in the latent space.

\begin{table}[!tp]
\centering
\begin{tabular}{lcccc}
\toprule
Algorithm & NLL & ELBO & recon. \\
\midrule
$\SO3$-MLP-s2s2             & 123.6 & 144.6 & 129.6  \\
$\SO3$-action-s2s2         & \textbf{46.90} & \textbf{63.35} & \textbf{48.35}  \\
$\mathcal{N}$-MLP 3-dim & 140.7 & 157.7 & 142.7 \\
$\mathcal{N}$-MLP 10-dim     & 64.02 & 80.80 & 65.80 \\
$\mathcal{N}$-MLP 30-dim     & 55.7 & 74.37 & 59.37 \\ 

\bottomrule
\end{tabular}
\caption{Results on sphere-cube of $\SO3$ encodings and $\mathbb{R}^N$ embedding encodings. The $\SO3$ models employ both regularizers, the $\mathbb{R}^N$ models neither. This achieved the best respective performance.}
\label{tab:vae-2}
\end{table}

\begin{table}[!ht]
\centering
\begin{tabular}{lcccc}
\toprule
Algorithm & NLL & ELBO & recon. \\
\midrule
$\SO3$-MLP-q         & 111.8 & 140.1 & 135.1  \\
$\SO3$-MLP-alg       & 218.9 & 316.7 & 301.7  \\
$\SO3$-MLP-s2s1      & 106.0 & 144.5 & 129.5  \\
$\SO3$-MLP-s2s2      & 123.6 & 144.6 & 129.6  \\
$\SO3$-action-q      & 378.6 & 471.2 & 456.2  \\
$\SO3$-action-alg    & 241.2 & 333.2 & 318.2  \\
$\SO3$-action-s2s1   & 128.5 & 173.0 & 158.0  \\
$\SO3$-action-s2s2   & \textbf{46.90} & \textbf{63.35} & \textbf{48.35}  \\
\bottomrule
\end{tabular}
\caption{Summary of results comparing $\SO3$ mean parameterization and model decoder on sphere-cubes. The models employ both regularizers.}
\label{tab:vae}
\end{table}

\section{Discussion \& Conclusion} \label{sec:conlusion}
In this paper we explored the use of manifold-valued latent variables, by proposing an extension of the reparameterization trick to compact connected Lie groups. We worked out the implementation details for the specific case of $\SO3$, and highlighted the various subtleties that must be taken into account to ensure a successful parameterization of the VAE. Through a series of experiments, we showed the importance of matching the topology of the latent data manifold with that of the latent variables to induce a continuous, well-behaved latent space. Additionally we demonstrated the improvement in learned latent space structure by using a group action decoder, and the need for care in choosing an embedding space for the posterior distribution's mean parameter.

We believe that the use of $\SO3$ and other well-known manifold-valued latent variables could present an interesting addition to tackling problems in such fields as model based RL and computer vision. Moving forward we thus aim to extend this theory to other Lie groups such as $\operatorname{SE}(3)$. A limitation of the current work, and reparameterizing distributions on specific manifolds in general, is that it relies on the assumption of \emph{a priori} knowledge about the observed data's latent structure. Hence in future work our ambition is to find a general theory to learn arbitrary manifolds not known in advance.  

\section*{Acknowledgements}
The authors would like to thank Rianne van den Berg, Jakub Tomczak, and Yvan Scher for their suggestions and support in improving this paper.

\clearpage
\newpage 

\bibliography{main}

\begin{thebibliography}{28}
\providecommand{\natexlab}[1]{#1}
\providecommand{\url}[1]{\texttt{#1}}
\expandafter\ifx\csname urlstyle\endcsname\relax
  \providecommand{\doi}[1]{doi: #1}\else
  \providecommand{\doi}{doi: \begingroup \urlstyle{rm}\Url}\fi

\bibitem[Berg et~al.(2018)Berg, Hasenclever, Tomczak, and
  Welling]{berg2018sylvester}
Berg, Rianne van~den, Hasenclever, Leonard, Tomczak, Jakub~M, and Welling, Max.
\newblock Sylvester normalizing flows for variational inference.
\newblock \emph{UAI}, 2018.

\bibitem[Burda et~al.(2016)Burda, Grosse, and
  Salakhutdinov]{burda2015importance}
Burda, Yuri, Grosse, Roger, and Salakhutdinov, Ruslan.
\newblock Importance weighted autoencoders.
\newblock \emph{ICLR}, 2016.

\bibitem[Burgess et~al.(2018)Burgess, Higgins, Pal, Matthey, Watters,
  Desjardins, and Lerchner]{burgess2018understanding}
Burgess, Christopher~P, Higgins, Irina, Pal, Arka, Matthey, Loic, Watters,
  Nick, Desjardins, Guillaume, and Lerchner, Alexander.
\newblock Understanding disentangling in $beta $-vae.
\newblock \emph{arXiv preprint arXiv:1804.03599}, 2018.

\bibitem[Chirikjian(2010)]{chirikjian2010information}
Chirikjian, Gregory~S.
\newblock Information-theoretic inequalities on unimodular lie groups.
\newblock \emph{Journal of geometric mechanics}, 2\penalty0 (2):\penalty0 119,
  2010.

\bibitem[Chirikjian(2012)]{Chirikjian2012-gr}
Chirikjian, Gregory~S.
\newblock \emph{Stochastic Models, Information Theory, and Lie Groups}.
\newblock 2012.

\bibitem[Chirikjian \& Kyatkin(2000)Chirikjian and
  Kyatkin]{chirikjian2000engineering}
Chirikjian, Gregory~S and Kyatkin, Alexander~B.
\newblock \emph{Engineering applications of noncommutative harmonic analysis:
  with emphasis on rotation and motion groups}.
\newblock CRC press, 2000.

\bibitem[Chirikjian \& Kyatkin(2016)Chirikjian and Kyatkin]{Chirikjian2016-ch}
Chirikjian, Gregory~S and Kyatkin, Alexander~B.
\newblock \emph{Harmonic Analysis for Engineers and Applied Scientists: Updated
  and Expanded Edition}.
\newblock Courier Dover Publications, July 2016.

\bibitem[Cohen \& Welling(2016)Cohen and Welling]{cohen2016group}
Cohen, Taco and Welling, Max.
\newblock Group equivariant convolutional networks.
\newblock In \emph{ICML}, pp.\  2990--2999, 2016.

\bibitem[Cohen \& Welling(2015)Cohen and Welling]{cohen2015harmonic}
Cohen, Taco~S and Welling, Max.
\newblock Harmonic exponential families on manifolds.
\newblock \emph{ICML}, 2015.

\bibitem[Cohen \& Welling(2017)Cohen and Welling]{cohen2016steerable}
Cohen, Taco~S and Welling, Max.
\newblock Steerable cnns.
\newblock \emph{ICLR}, 2017.

\bibitem[Cohen et~al.(2018{\natexlab{a}})Cohen, Geiger, Köhler, and
  Welling]{s.2018spherical}
Cohen, Taco~S., Geiger, Mario, Köhler, Jonas, and Welling, Max.
\newblock Spherical {CNN}s.
\newblock \emph{ICLR}, 2018{\natexlab{a}}.

\bibitem[Cohen et~al.(2018{\natexlab{b}})Cohen, Geiger, and
  Weiler]{Cohen2018-ex}
Cohen, Taco~S, Geiger, Mario, and Weiler, Maurice.
\newblock Intertwiners between induced representations (with applications to
  the theory of equivariant neural networks).
\newblock March 2018{\natexlab{b}}.

\bibitem[Davidson et~al.(2018)Davidson, Falorsi, Cao, Kipf, and Tomczak]{s-vae}
Davidson, Tim~R., Falorsi, Luca, Cao, Nicola~De, Kipf, Thomas, and Tomczak,
  Jakub~M.
\newblock {Hyperspherical Variational Auto-Encoders}.
\newblock \emph{UAI}, 2018.

\bibitem[Eslami et~al.(2018)Eslami, Jimenez~Rezende, Besse, Viola, Morcos,
  Garnelo, Ruderman, Rusu, Danihelka, Gregor, Reichert, Buesing, Weber,
  Vinyals, Rosenbaum, Rabinowitz, King, Hillier, Botvinick, Wierstra,
  Kavukcuoglu, and Hassabis]{eslami1204}
Eslami, S. M.~Ali, Jimenez~Rezende, Danilo, Besse, Frederic, Viola, Fabio,
  Morcos, Ari~S., Garnelo, Marta, Ruderman, Avraham, Rusu, Andrei~A.,
  Danihelka, Ivo, Gregor, Karol, Reichert, David~P., Buesing, Lars, Weber,
  Theophane, Vinyals, Oriol, Rosenbaum, Dan, Rabinowitz, Neil, King, Helen,
  Hillier, Chloe, Botvinick, Matt, Wierstra, Daan, Kavukcuoglu, Koray, and
  Hassabis, Demis.
\newblock Neural scene representation and rendering.
\newblock \emph{Science}, 360\penalty0 (6394):\penalty0 1204--1210, 2018.
\newblock ISSN 0036-8075.
\newblock \doi{10.1126/science.aar6170}.
\newblock URL \url{http://science.sciencemag.org/content/360/6394/1204}.

\bibitem[Figurnov et~al.(2018)Figurnov, Mohamed, and
  Mnih]{figurnov2018implicit}
Figurnov, Michael, Mohamed, Shakir, and Mnih, Andriy.
\newblock Implicit reparameterization gradients.
\newblock \emph{arXiv preprint arXiv:1805.08498}, 2018.

\bibitem[Hall(2003)]{hall2003lie}
Hall, B.
\newblock \emph{Lie Groups, Lie Algebras, and Representations: An Elementary
  Introduction}.
\newblock Graduate Texts in Mathematics. Springer, 2003.
\newblock ISBN 9780387401225.

\bibitem[Jang et~al.(2017)Jang, Gu, and Poole]{Jang2016CategoricalRW-gumbel}
Jang, Eric, Gu, Shixiang, and Poole, Ben.
\newblock Categorical reparameterization with gumbel-softmax.
\newblock \emph{ICLR}, abs/1611.01144, 2017.

\bibitem[Kingma \& Welling(2013)Kingma and Welling]{KingmaW13-vae}
Kingma, Diederik~P. and Welling, Max.
\newblock Auto-encoding variational bayes.
\newblock \emph{CoRR}, abs/1312.6114, 2013.

\bibitem[Kingma et~al.(2016)Kingma, Salimans, Jozefowicz, Chen, Sutskever, and
  Welling]{kingma2016improved}
Kingma, Diederik~P, Salimans, Tim, Jozefowicz, Rafal, Chen, Xi, Sutskever,
  Ilya, and Welling, Max.
\newblock Improved variational inference with inverse autoregressive flow.
\newblock In \emph{NIPS}, pp.\  4743--4751, 2016.

\bibitem[Maddison et~al.(2017)Maddison, Mnih, and Teh]{maddison2016concrete}
Maddison, Chris~J, Mnih, Andriy, and Teh, Yee~Whye.
\newblock The concrete distribution: A continuous relaxation of discrete random
  variables.
\newblock \emph{ICLR}, 2017.

\bibitem[Naesseth et~al.(2017)Naesseth, Ruiz, Linderman, and
  Blei]{naesseth2017reparameterization}
Naesseth, Christian, Ruiz, Francisco, Linderman, Scott, and Blei, David.
\newblock Reparameterization gradients through acceptance-rejection sampling
  algorithms.
\newblock In \emph{AISTATS}, pp.\  489--498, 2017.

\bibitem[Nalisnick \& Smyth(2017)Nalisnick and Smyth]{stick}
Nalisnick, Eric and Smyth, Padhraic.
\newblock Stick-breaking variational autoencoders.
\newblock \emph{ICLR}, 2017.

\bibitem[Rezende \& Mohamed(2015)Rezende and Mohamed]{normalizing-flows}
Rezende, Danilo and Mohamed, Shakir.
\newblock Variational inference with normalizing flows.
\newblock \emph{ICML}, 37:\penalty0 1530--1538, 2015.

\bibitem[Rezende et~al.(2014)Rezende, Mohamed, and
  Wierstra]{rezende2014stochastic}
Rezende, Danilo~Jimenez, Mohamed, Shakir, and Wierstra, Daan.
\newblock Stochastic backpropagation and approximate inference in deep
  generative models.
\newblock \emph{ICML}, pp.\  1278--1286, 2014.

\bibitem[Rodrigues(1840)]{rodrigues1840lois}
Rodrigues, Olinde.
\newblock \emph{Des lois g{\'e}om{\'e}triques qui r{\'e}gissent les
  d{\'e}placements d'un syst{\`e}me solide dans l'espace: et de la variation
  des cordonn{\'e}es provenant de ces d{\'e}placements consid{\'e}r{\'e}s
  ind{\'e}pendamment des causes qui peuvent les produire}.
\newblock 1840.

\bibitem[Tomczak \& Welling(2017)Tomczak and Welling]{vamp-prior}
Tomczak, Jakub~M and Welling, Max.
\newblock {VAE with a VampPrior}.
\newblock \emph{AISTATS}, 2017.

\bibitem[Weiler et~al.(2018)Weiler, Hamprecht, and
  Storath]{Weiler2018-STEERABLE}
Weiler, Maurice, Hamprecht, Fred~A, and Storath, Martin.
\newblock Learning steerable filters for rotation equivariant {CNNs}.
\newblock In \emph{CVPR}, 2018.

\bibitem[Worrall et~al.(2017)Worrall, Garbin, Turmukhambetov, and
  Brostow]{Worrall2017-HNET}
Worrall, Daniel~E, Garbin, Stephan~J, Turmukhambetov, Daniyar, and Brostow,
  Gabriel~J.
\newblock Harmonic networks: Deep translation and rotation equivariance.
\newblock In \emph{{CVPR}}, 2017.

\end{thebibliography}
\bibliographystyle{icml2018}

\newpage
\newpage
\onecolumn

\appendix

\section{Additional Figures} \label{ap:additional-figures}

\begin{figure*}[!ht]
\centering
\subfigure[Ground truth]{\includegraphics[height=0.15\textwidth]{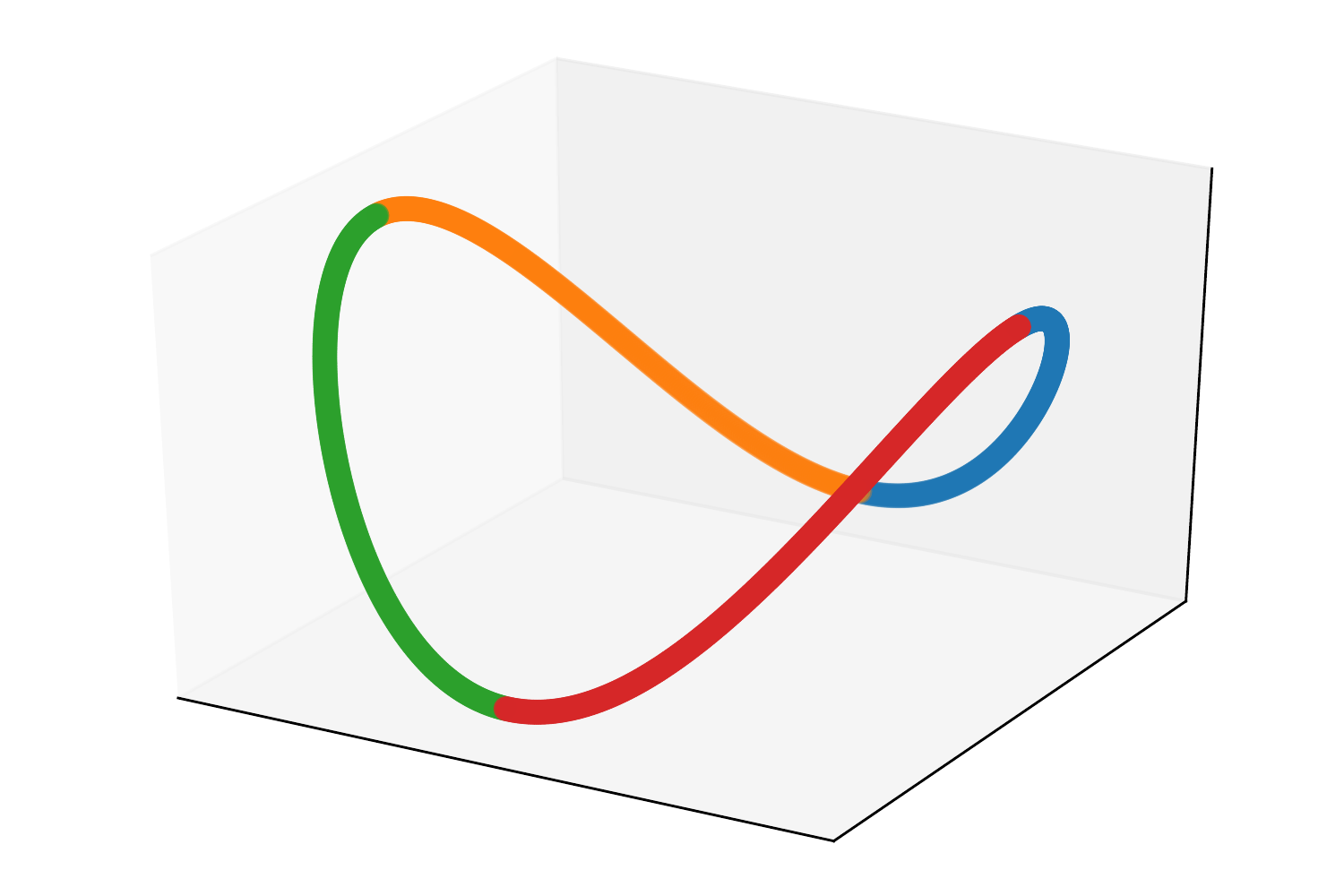}}
\subfigure[$\SO3$-action-s2s2]{\includegraphics[height=0.15\textwidth]{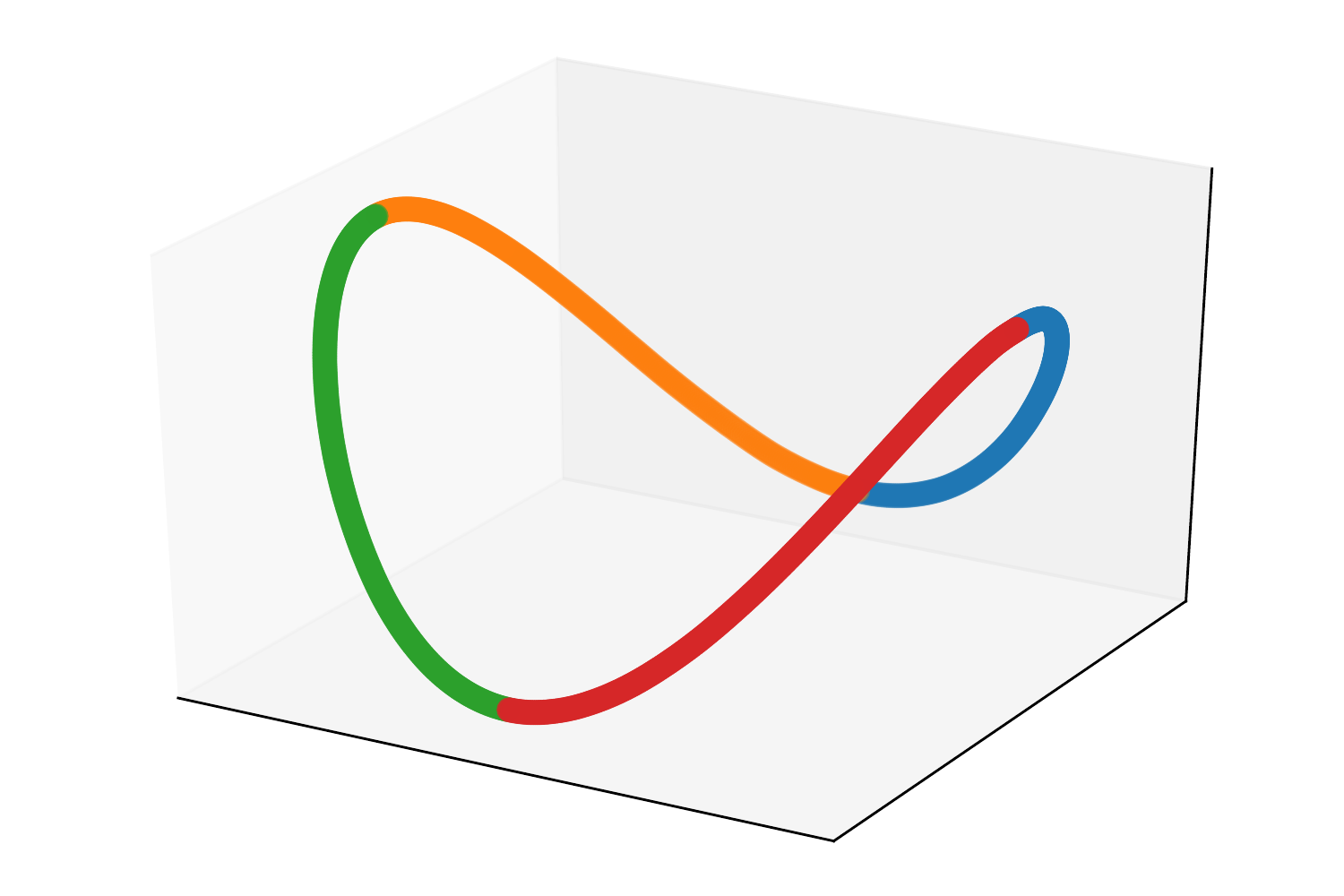}}
\subfigure[$\SO3$-action-alg]{\includegraphics[height=0.15\textwidth]{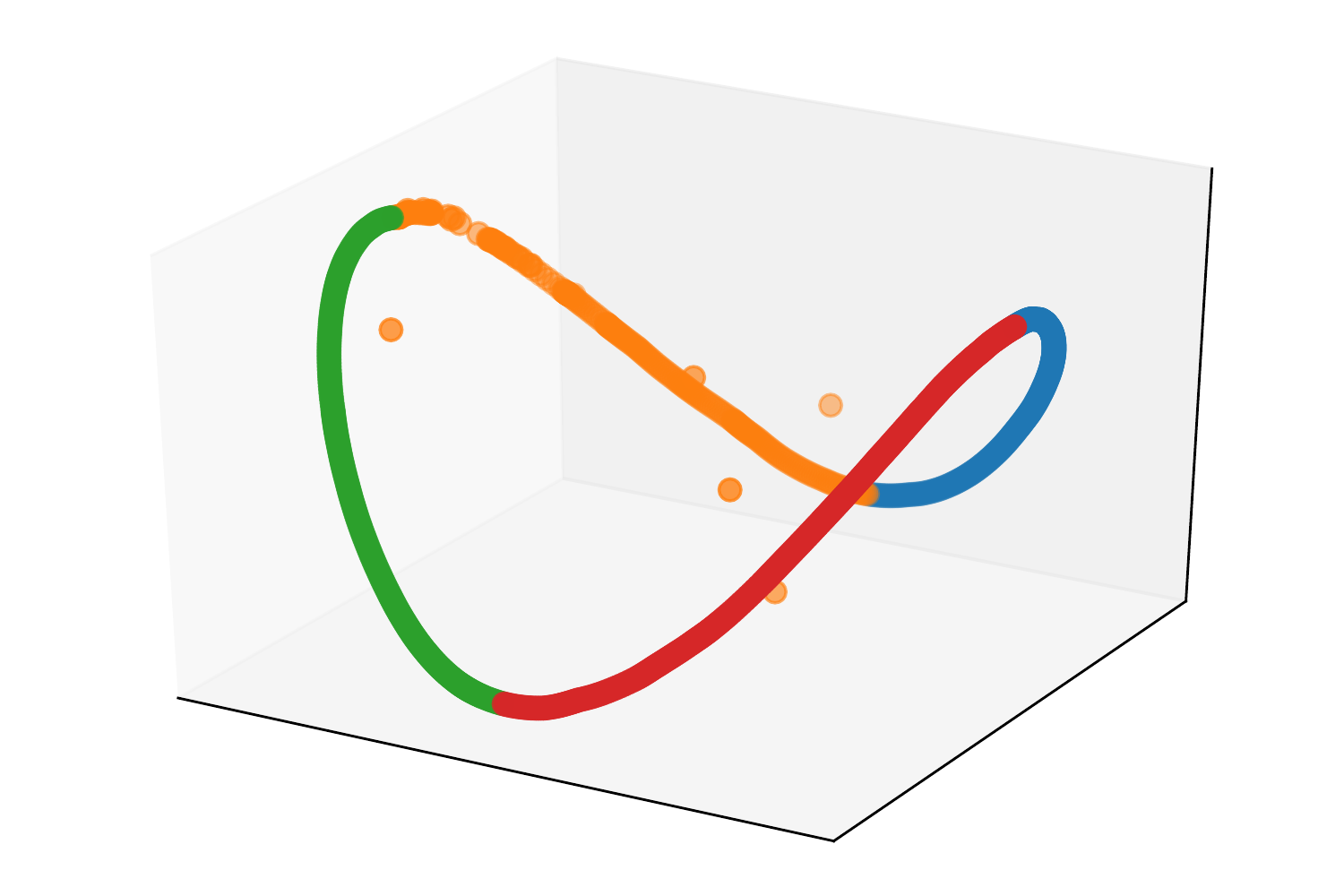}}
\subfigure[$\SO3$-action-q]{\includegraphics[height=0.15\textwidth]{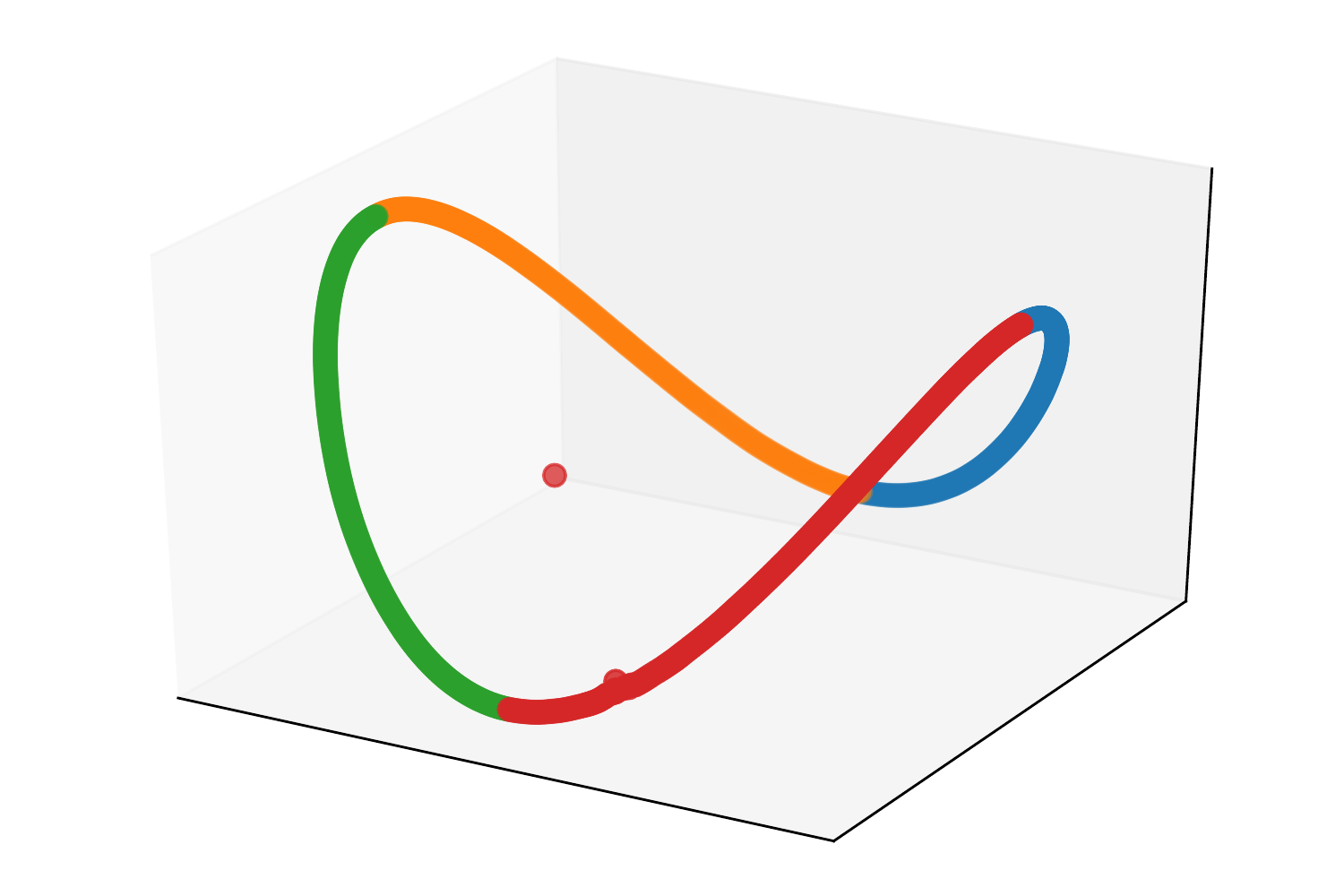}}
\subfigure[$\SO3$-action-s2s1]{\includegraphics[height=0.15\textwidth]{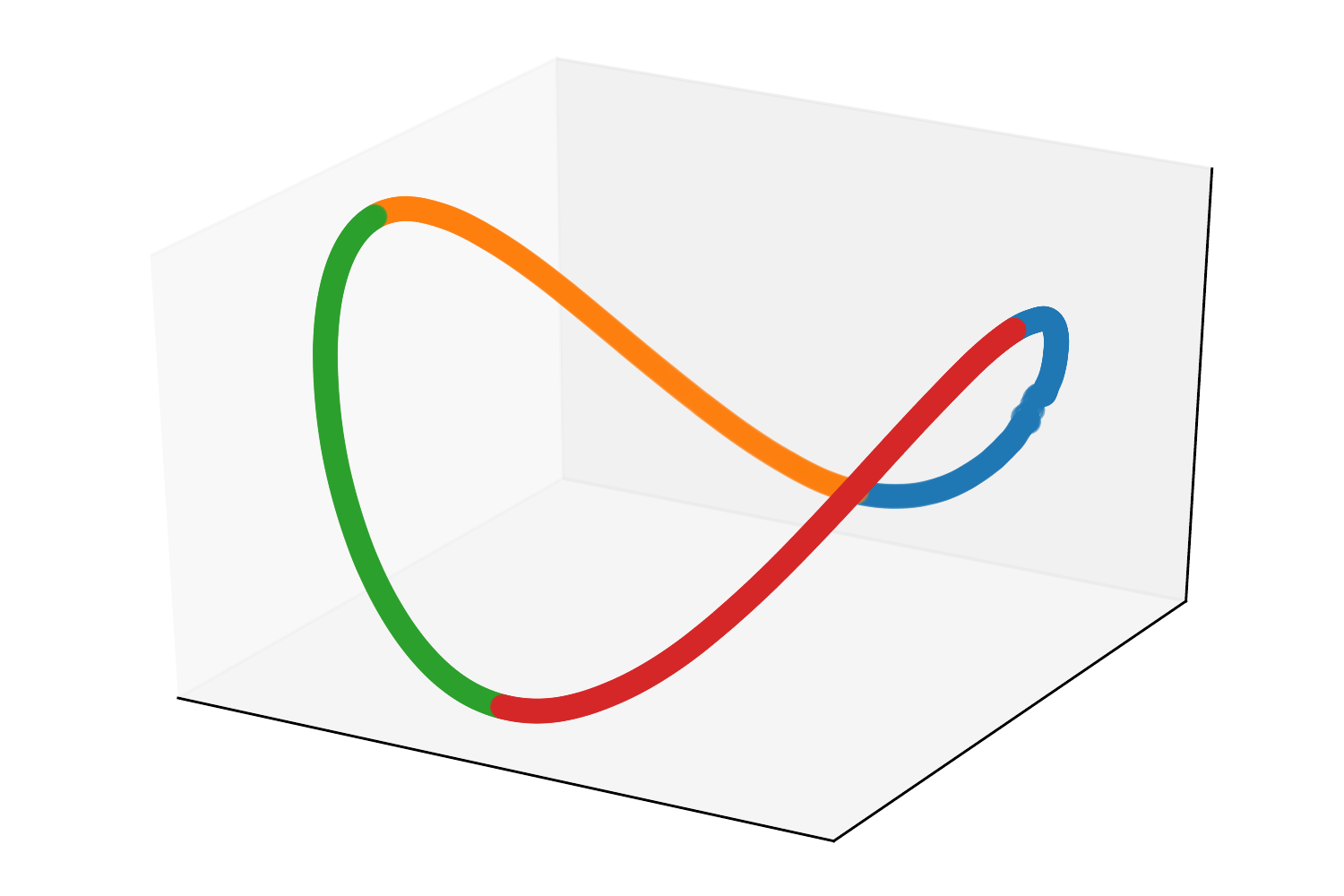}}
\subfigure[$\mathcal{N}$-action-3-dim]{\includegraphics[height=0.15\textwidth]{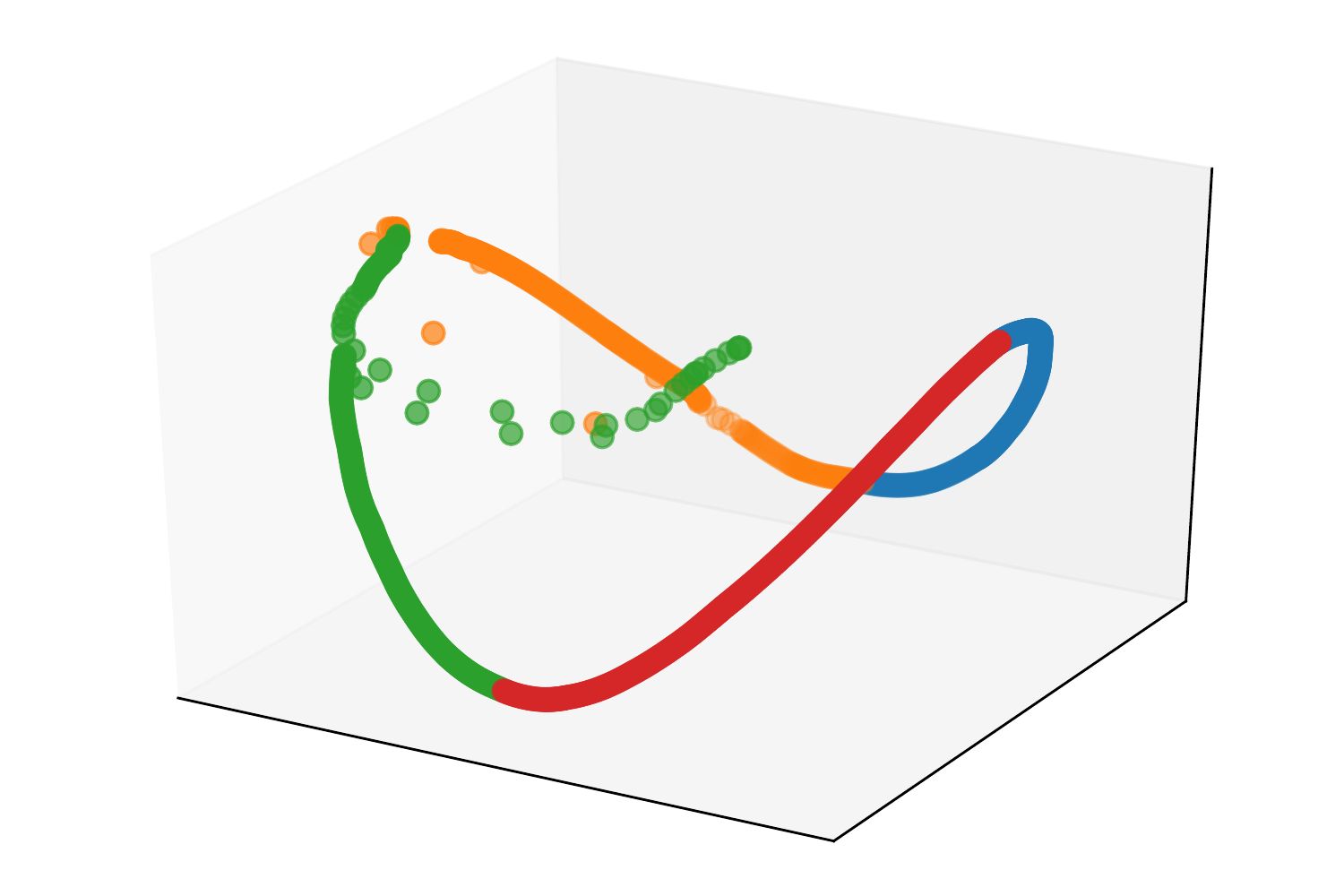}}
\subfigure[$\mathcal{S}$-action-3-dim]{\includegraphics[height=0.15\textwidth]{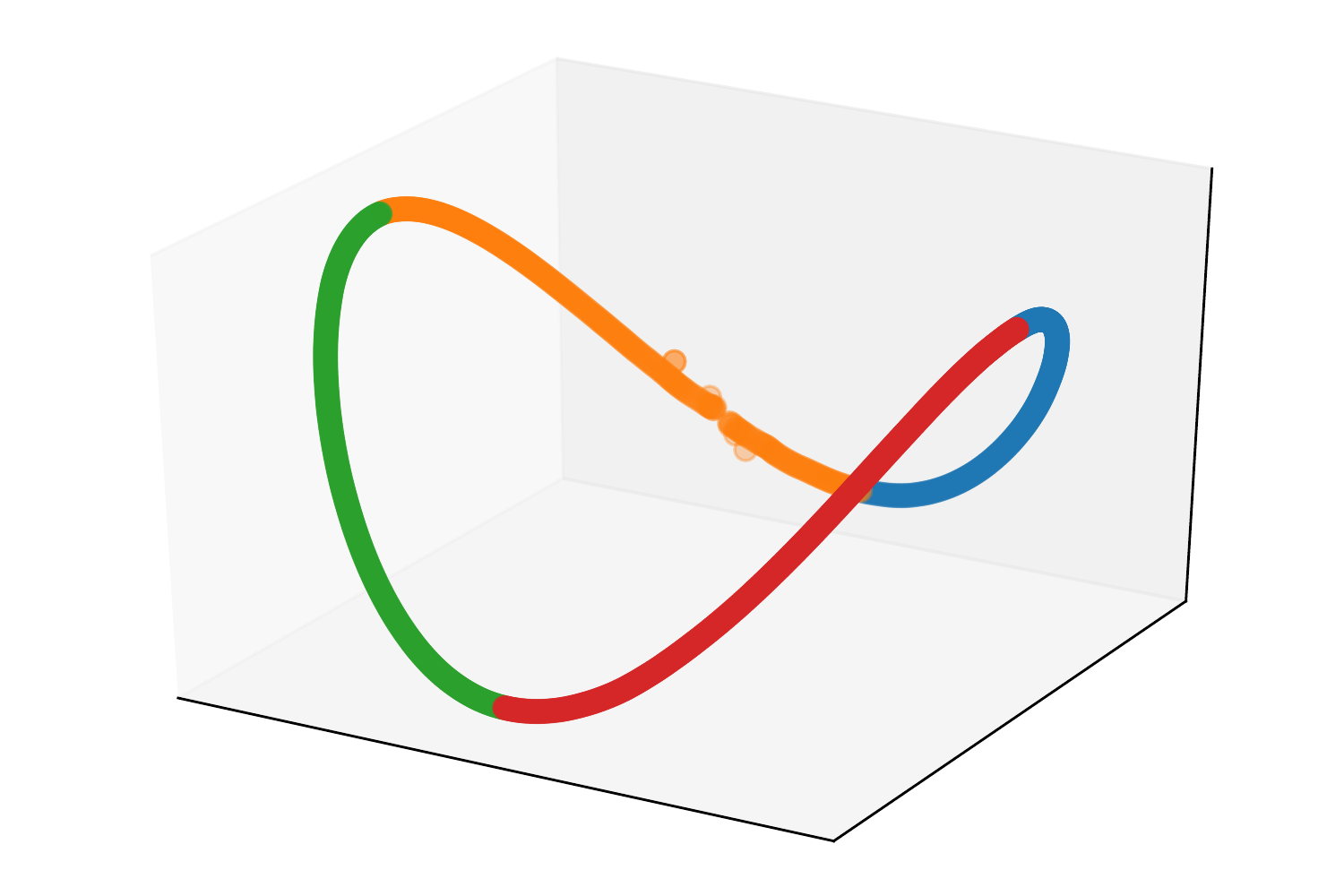}}
\caption{Reconstructions of a $S^1$ trajectory in the Toy data set. The $\mathbb{R}^{64}$ elements are mapped to 3D by Principal Component Analysis. See Section \ref{sec:toy} for details.}
\label{fig:toy-reconstruction}
\end{figure*}

\begin{figure*}[!ht]
\centering
\subfigure[$\SO3$-action-s2s2]{\includegraphics[height=0.15\textwidth]{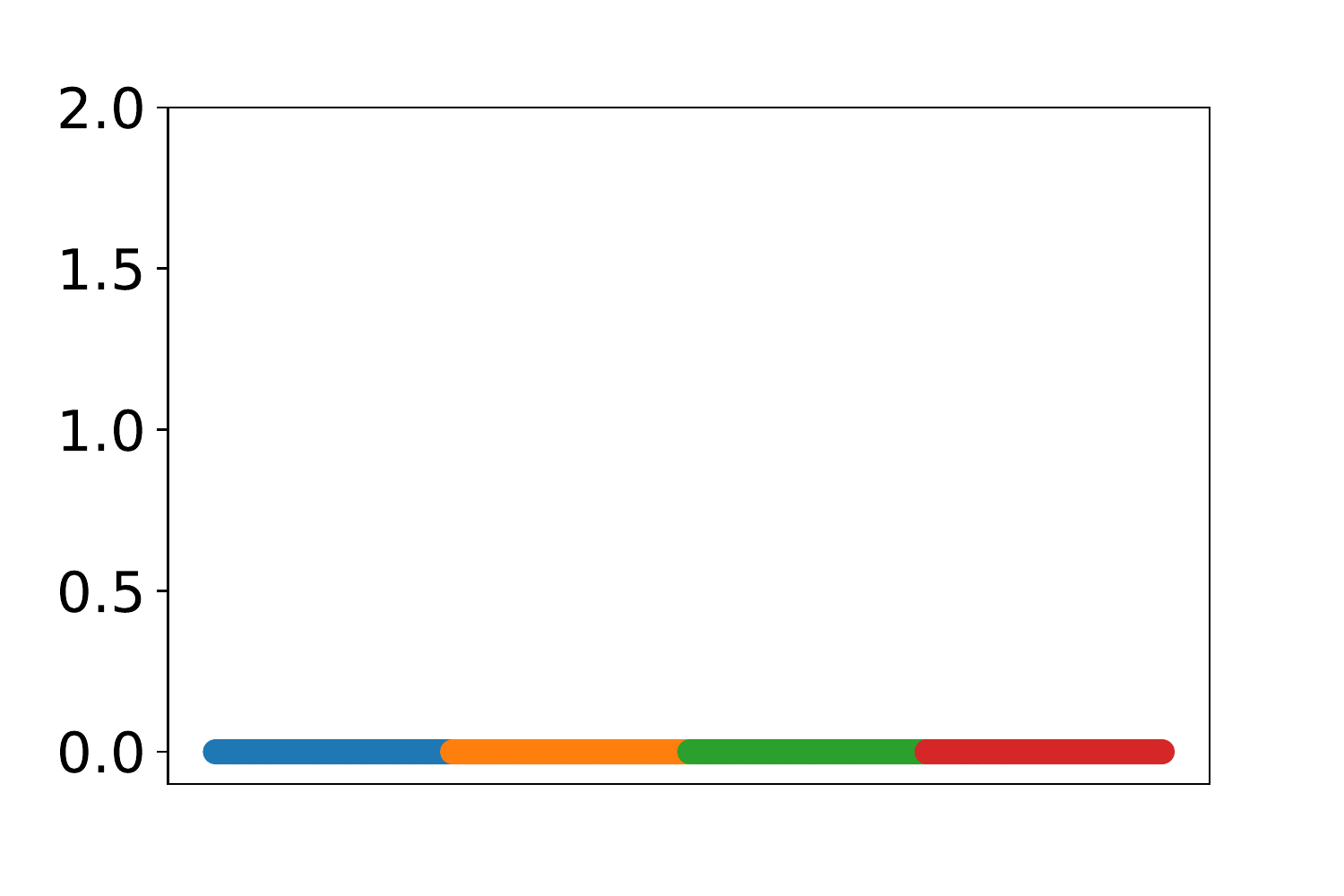}}
\subfigure[$\SO3$-action-alg]{\includegraphics[height=0.15\textwidth]{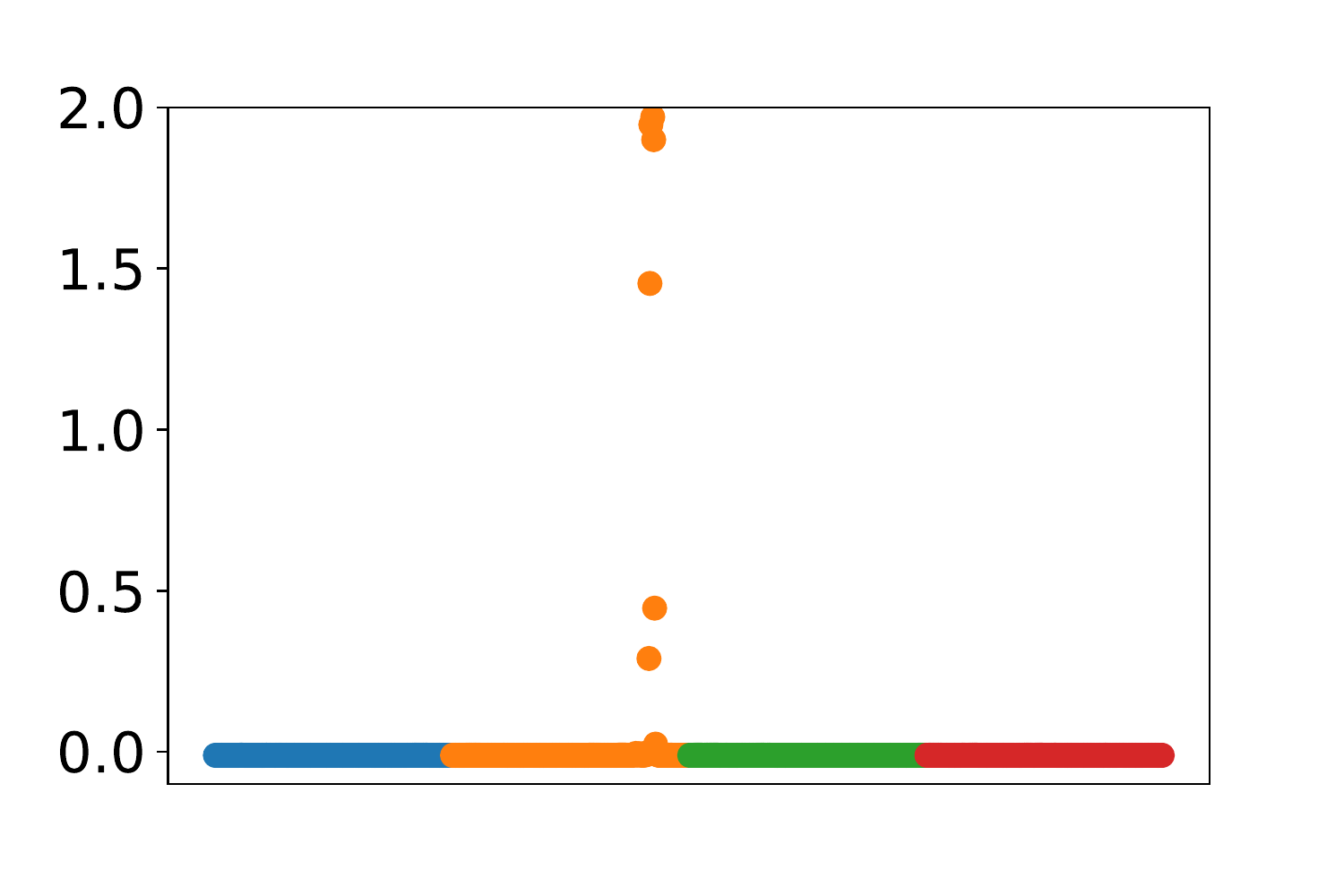}}
\subfigure[$\SO3$-action-q]{\includegraphics[height=0.15\textwidth]{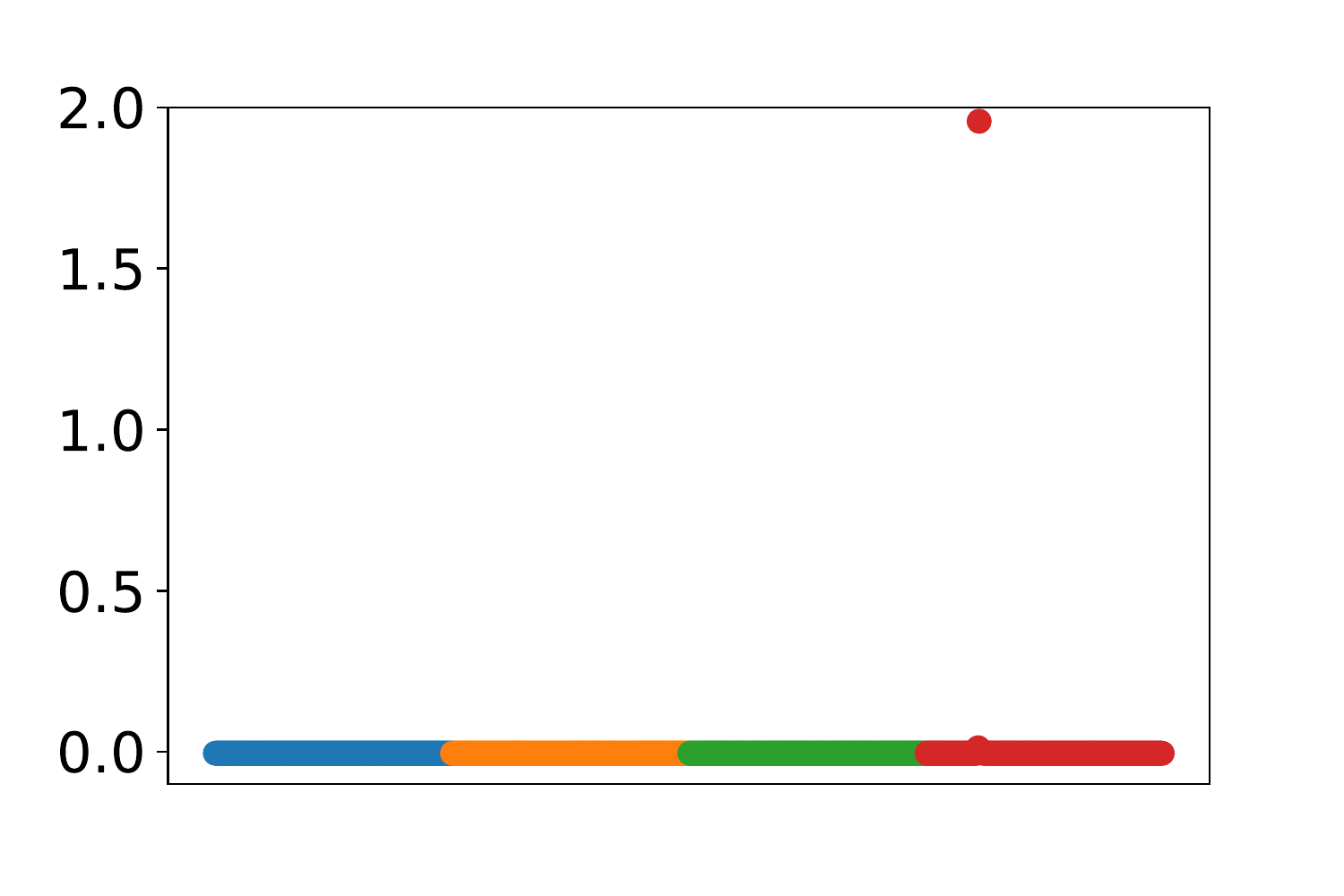}}
\\
\subfigure[$\SO3$-action-s2s1]{\includegraphics[height=0.15\textwidth]{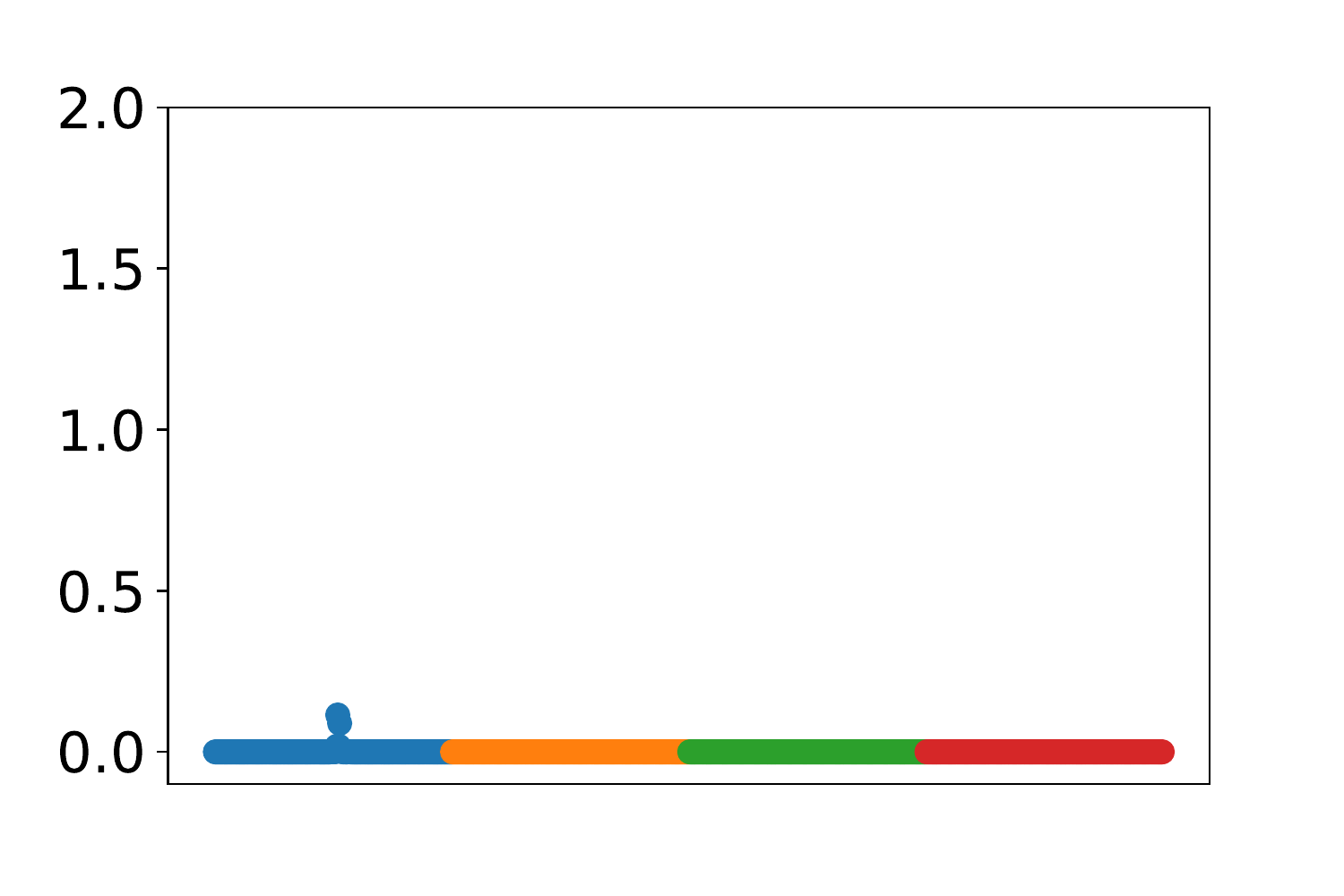}}
\subfigure[$\mathcal{N}$-action-3-dim]{\includegraphics[height=0.15\textwidth]{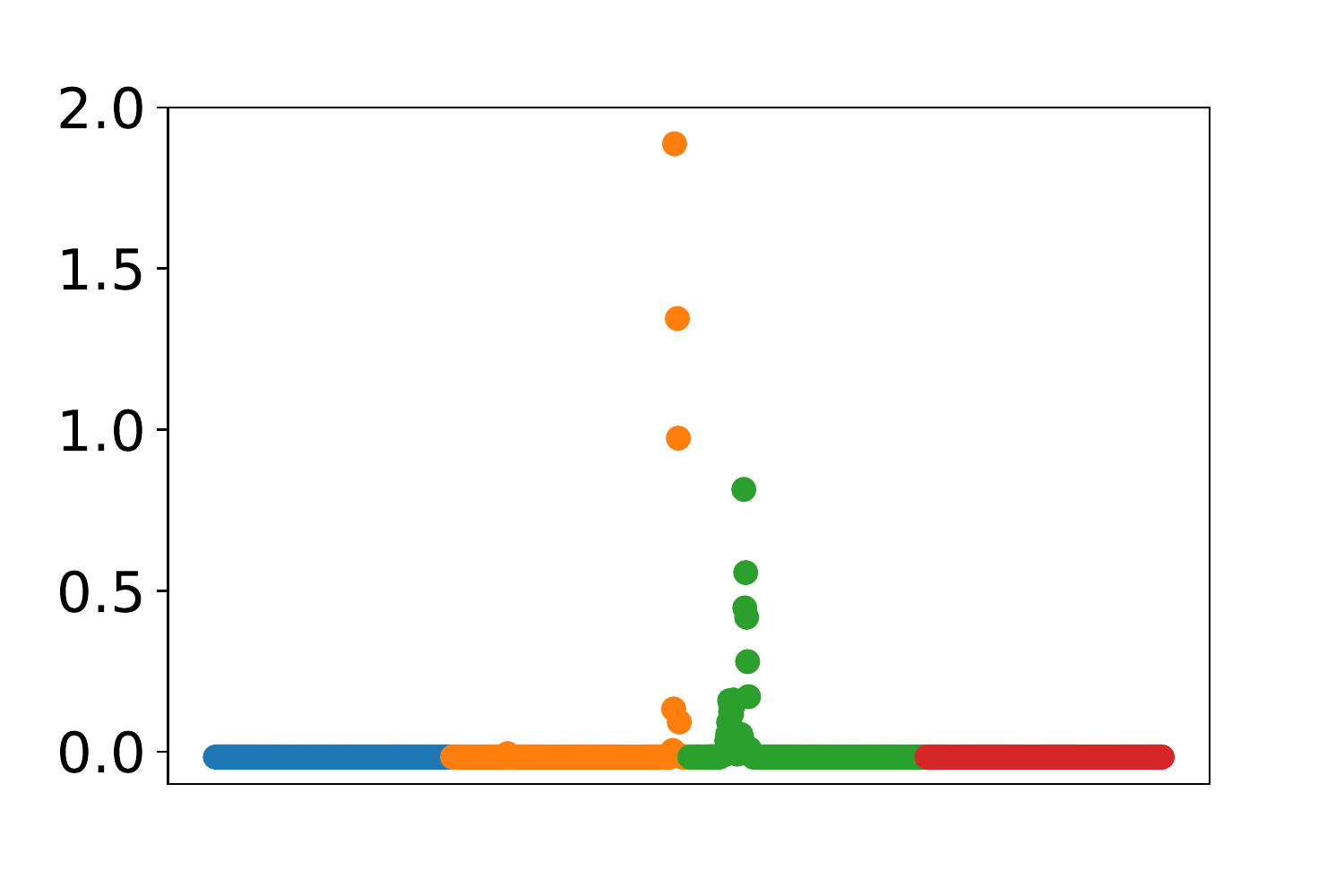}}
\subfigure[$\mathcal{S}$-action-3-dim]{\includegraphics[height=0.15\textwidth]{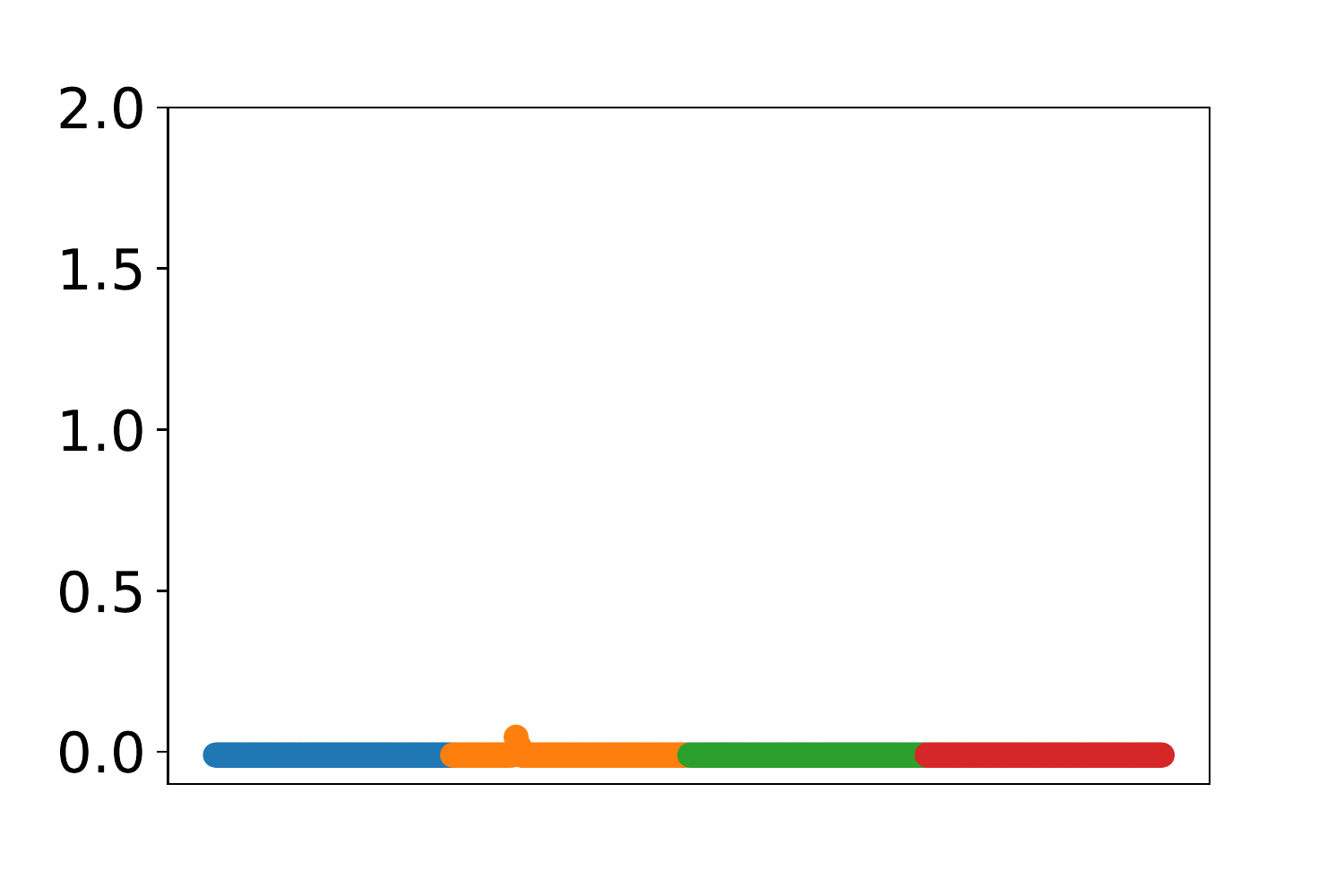}}
\caption{Discontinuities in the latent space along a $S^1$ trajectory in Toy data set. Shown is $\|f(x_{i+1})-f(x_i)\|^2$ for encoder $f$ along the trajectory. See Section \ref{sec:toy} for details.}
\label{fig:toy-discontinuity}
\end{figure*}

\section{Pushforward Measure $\SO3$} \label{appendix:pushforward-so3}

\begin{customthm}{1}
Let $(\mathbb{R}^3,\lambda, \mathcal{B}[\mathbb{R}^3] )$ the real space provided with the Lebesgue measure on the Borel algebra on $\mathbb{R}^3$. Let $(\SO3,\nu,\mathcal{B}[\SO3])$ The group of 3 dimensional rotations provided with the normalized Haar measure $\nu$ on the Borel algebra on $\SO3$. Consider then the probability measure $\mu:\mathcal{B}[\mathbb{R}^3]\to [0,1]$ absolutely continuous w.r.t $\lambda$ with density $r$. Consider the exponential map $\exp: \mathbb{R}^3 \to \SO3$ that is differentiable thus continuous thus measurable.
Let then $\exp_*(\mu)$ the pushforward of $\mu$ by the $\exp$ function. then $\exp_*(\mu)$ is absolutely continuous with respect of the Haar measure $\nu$.
($\exp_*(\mu) \ll \nu$) 
\end{customthm}

\begin{proof}

Define the sets:
\begin{equation}
A_k = \{x\in \mathbb{R}^3 : \|x\| \in (k\pi,(k+1)\pi)\} \quad B_k = \{x\in \mathbb{R}^3 : \|x\| =  k\pi\} \quad k\in \mathbb{N}
\end{equation}
Then note that:
\begin{equation}
\mathbb{R}^3 = \left(\dot{\bigcup_{k\in \mathbb{N}}}A_k\right)\dot{\cup}\left(\dot{\bigcup_{k\in \mathbb{N}}}B_k\right)
\end{equation}
And since $m(B_k) = \mu(B_k) = 0$,  then $\mu_{B_k}(E) = 0 \quad \forall E\in \mathcal{B}[\mathbb{R}^3]\ k\in\mathbb{N}$ . Therefore:
\begin{equation}
\mu(E) = \sum_{k\in\mathbb{N}}\mu_{A_k}(E) \quad \forall E\in \mathcal{B}[\mathbb{R}^3]
\end{equation}
Then we have $\mu = \sum_{k\in\mathbb{N}}\mu_{A_k}$. \newline 
Now consider the pushforward measure $\exp_*(\mu)$, we then have that:
\begin{align}
(\exp_*(\mu))(E) = \mu(\exp^{-1}(E)) = \sum_{k\in\mathbb{N}}\mu_{A_k}(\exp^{-1}(E)) = 
\sum_{k\in\mathbb{N}}(\exp_{*}(\mu_{A_k}))(E) = \\
\sum_{k\in\mathbb{N}}((\exp_{|A_k})_*(\mu))(E) 
\quad \forall E\in \mathcal{B}[\SO3]
\end{align}
Then we have $\exp_*(\mu) = \sum_{k\in\mathbb{N}}\exp_{*}(\mu_{A_k}) = \sum_{k\in\mathbb{N}}((\exp_{|A_k})_*(\mu))$. Where $\exp_{|A_k}$ is the $\exp$ function restricted to $A_k$. Moreover notice that $\exp_{|A_k}$ is a injective, therefore we can apply the change of variable formula:
\begin{align}
((\exp_{|A_k})_*(\mu))(E) = \int_{\exp_{|A_k}^{-1}(E)} r \ d\lambda = \int_{E}(r\circ\exp_{|A_k}^{-1})\cdot|J_{\exp_{|A_k}^{-1}}| \ d\nu
\end{align}
Then $(\exp_{|A_k})_*(\mu)\ll \nu$ and since $\exp_*(\mu)= \sum_{k\in\mathbb{N}}((\exp_{|A_k})_*(\mu))$ then $\exp_*(\mu)\ll \nu$. \qedhere
\end{proof}
The proof then tells us how to compute the Radon-Nikodym derivative of the pushforward with respect to the Haar measure. In fact:
\begin{equation}
    \frac{d(\exp_{|A_k})_*(\mu)}{d \nu} = (r\circ\exp_{|A_k}^{-1})\cdot|J_{\exp_{|A_k}^{-1}}| \quad, \quad 
    \frac{d\exp_*(\mu)}{d \nu} = \sum_{k\in\mathbb{N}} \frac{d(\exp_{|A_k})_*(\mu)}{d \nu} 
\end{equation}

Defining $\hat q : = \frac{d\exp_*(\mu)}{d \nu}$ we then have:
\begin{equation}
\hat q(R) = \sum_{k\in\mathbb{N}} (r\circ\exp_{|A_k}^{-1}(R))\cdot|J_{\exp_{|A_k}^{-1}}(R)| = \sum_{v\in\exp^{-1}(R)}r(v)\cdot|J_{\exp}(v)|^{-1}
\end{equation}
From \cite{chirikjian2010information} we have that 
\begin{equation}
 |J_{\exp}(v)| = \frac{2 - 2\cos\|v\|}{\|v\|^2} 
\end{equation}
We then have: 
\begin{equation}
\hat q(R) =  \sum_{v\in\exp^{-1}(R)}r(v)\frac{\|v\|^2}{2 - 2\cos(\|v\|)} 
\end{equation}

To then have an expression explicitly dependent on $R$ consider that 
\begin{equation}
\exp^{-1}_{|A_k}(R) = \frac{\exp^{-1}_{|A_0}(R)}{\|\exp^{-1}_{|A_0}(R)\|}(\|\exp^{-1}_{|A_0}(R)\| + 2k\pi) = \frac{\log(R)}{\|\log(R)\|}(\|\log(R)\| + k\pi) \quad \text{if } k \text{ is even}
\end{equation}
\begin{equation}
\exp^{-1}_{|A_k}(R) = \frac{\exp^{-1}_{|A_0}(R)}{\|\exp^{-1}_{|A_0}(R)\|}(\|\exp^{-1}_{|A_0}(R)\| + 2k\pi) = \frac{\log(R)}{\|\log(R)\|}(\|\log(R)\| + -(k+1)\pi) \quad \text{if } k \text{ is odd}
\end{equation}
Where we have defined $\log:= \exp^{-1}_{|A_0}(R)$. Moreover we then have:
\begin{equation}
|J_{\exp_{|A_k}^{-1}}(R)| = \frac{\|\exp^{-1}_{|A_k}(R)\|^2}{2 - 2\cos(\|\exp^{-1}_{|A_k}(R)\|)} = \frac{(\|\log(R)\| + k\pi)^2}{2 - 2\cos(\|\log(R)\|)} \quad \text{if } k \text{ is even}
\end{equation}
\begin{equation}
|J_{\exp_{|A_k}^{-1}}(R)| = \frac{\|\exp^{-1}_{|A_k}(R)\|^2}{2 - 2\cos(\|\exp^{-1}_{|A_k}(R)\|)} = \frac{(\|\log(R)\| - (k+1)\pi)^2}{2 - 2\cos(\|\log(R)\|)} \quad \text{if } k \text{ is odd}
\end{equation}
Putting everything together:
\begin{equation}\label{eq:post2}
\hat q(R) =  \sum_{k\in\mathbb{Z}}r\left(\frac{\log(R)}{\|\log(R)\|}(\|\log(R)\| + 2k\pi)\right)\frac{(\|\log(R)\| + 2k\pi)^2}{2 - 2\cos(\|\log(R)\|)}
\end{equation}    

Where from \cite{chirikjian2010information} we have:

\begin{equation}
 \log(R) = \frac{\theta(R)}{2 \sin(\theta(R))} \left(R - R^\top\right) \quad 
 \theta({R}) = \cos^{-1}\left(\frac{\text{tr}(R)-1}{2}\right)
\end{equation}
This gives us the final expression:
\begin{equation}
\hat q(R|\sigma) = \sum_{k \in \mathbb{Z}} r\left(\frac{\log(R)}{\theta(R)}(\theta(R) + 2k\pi)\right)\frac{(\theta(R) + 2k\pi)^2}{3 - \text{tr}({R})}
\end{equation}

\section{Entropy computation}
We oprimize MC estimates of the Entropy:
\begin{equation}
\mathbb{H}(q) = \mathbb{H}(\hat{q}) \approx -\frac{1}{N}\sum_{i=1}^N \log \hat{q}(R_i), \quad R_i \sim \hat{q}(R_i) \nonumber
\end{equation}
(Where we dropped dependency on the parameters for simplicity)
Then using, Equation \eqref{eq:post2}:
\begin{equation}
\mathbb{H}(q) = \mathbb{H}(\hat{q}) \approx -\frac{1}{N}\sum_{i=1}^N \log \sum_{k\in\mathbb{Z}}r\left(\frac{\log(R_i)}{\|\log(R_i)\|}(\|\log(R_i)\| + 2k\pi)\right)\frac{(\|\log(R_i)\| + 2k\pi)^2}{2 - 2\cos(\|\log(R_i)\|)} , \quad R_i \sim \hat{q}(R_i) \nonumber
\end{equation}
In the way we defined $\hat q$ we obtain samples from it in the following way:
\begin{equation}
    R_i = \exp(\vv_i) \quad \vv_i \sim r(\vv_i)
\end{equation}
Substituting it in in the previous expression we get:
\begin{align}
            \mathbb{H}(q)&\approx -\frac{1}{N}\sum_{i=1}^N \log  \hat{q}(\exp(\vv_i)) \nonumber
            \\
            &= -\frac{1}{N}\sum_{i=1}^N \log \sum_{k \in \mathbb{Z}} r\Bigl(\frac{\vv_i}{\|\vv_i\|} (\|\vv_i\|+ 2k\pi)\Bigl) \cdot \nonumber \frac{(\|\vv_i\|+ 2k\pi)^2 }{2 - 2\cos(\|\vv_i\|)},
\quad \vv_i \sim r(\vv_i)
\end{align}
Notice that this expression depends only on the samples from $r$ in the lie algebra

Assuming the density $r$ decays quickly enough to zero, the above infinite summation can be truncated. This is always the case for exponentially decaying distributions, like the Normal. 
The truncated summation can then can be computed using the {\it logsumexp} trick:

\begin{align}
            \mathbb{H}(q) \approx -\frac{1}{N}\sum_{i=1}^N \text{logsumexp}_k\left( \log r\Bigl(\frac{\vv_i}{\|\vv_i\|} (\|\vv_i\|+ 2k\pi)\Bigl) + \log \frac{(\|\vv_i\|+ 2k\pi)^2 }{2 - 2\cos(\|\vv_i\|)}\right),
\quad \vv_i \sim r(\vv_i)
\end{align}

\section{Mean parameterization} \label{ap:mean-param}
As discussed above, some requirements exist on $\pi: \mathcal{Y} \to \SO3$ for the encoder $\text{enc}^mu$ to correctly represent the data manifold.

We split $\pi$ in the composition of $\phi : \mathcal{Y} \to \mathcal{Y}'$ and $\psi : \mathcal{Y}' \to \SO3$, both generally discontinuous. We assume $\mathcal{Y}=\mathbb{R}^m$ to be a neural network output. The functions $\psi$ below are known ways to surjectively map to $\SO3$. $\phi$ are constructed to map from $\mathbb{R}^m$ to the domain of $\psi$.

We discuss the existence of a map $i: \SO3 \to \mathcal{Y}'$ such that it is a right inverse of $\psi$ ($\psi \circ i = \text{id}_{\SO3}$), which is necessary for the correct encoder to exist.

\begin{enumerate}
    \item {\bf Algebra} with $\mathcal{Y}=\mathcal{Y}'=\mathbb{R}^3$, $\phi=\text{id}$. This method simply uses the exponential map:
    \begin{align}
        \pi: \mathbb{R}^3 &\rightarrow \SO3\\
        \vv &\mapsto \exp(\vv_{\times})
    \end{align}
    It's inverses are the branches of the log map. However, a path in $\SO3$ that is a full rotation around a fixed axis is continuous in $\SO3$ but discontinuous in the algebra, when mapped with the log map. Thus the log map is not continuous.
    \item {\bf Quaternions} with $\mathcal{Y}=\mathbb{R}^4, \mathcal{Y}'=S^3$, $\phi(x)=x / \|x\|$. The unit Quaternions, which are homeomorphic to $S^3$ are a 'double cover' of $\SO3$, which means that a continuous surjective projection $\pi : S^3 \to \SO3$ exists that is two-to-one. The projection map can be found in \citet[Eqn. (5.60)]{chirikjian2000engineering}. Using the theory of Fiber Bundles (recognizing $S^3$ as a non-trivial principle bundle with base space $\SO3$), one can show that no embedding $i$ exists.
    
    \item {\bf s2s1}($S^2\times S^1$) with $\mathcal{Y}=\mathbb{R}^3 \times \mathbb{R}^2, \mathcal{Y}'=S^2 \times S^1$, $\phi(x, y)=(x / \|x\|,y / \|y\|) $. This is the map from an axis in $S^2$ and angle in $S^1$.
    \begin{align}
        \pi: \mathcal{S}^2\times \mathcal{S} &\rightarrow \SO3\\
       (\mathbf{u}, \vv) &\mapsto \mathbf{I} + v_2\mathbf{u}_\times + 
    (1 - v_1)\mathbf{u}_\times^2
    \end{align}
    For $i$, to be continuous, its image $i(\SO3)$ must be closed (as it is a compact subset of a Hausdorff space). Thus so must the set $A=i(\SO3) \cap S^2 \times \{\pi\}$. However, as $\psi(\mu, \pi)=\psi(-\mu,\pi)$ for $\mu \in S^2$, $A$ is a hemisphere (times a point) that does not contain its entire boundary, thus it is not closed and $i$ is not continuous.
    
    \item {\bf s2s2}($S^2\times S^2$) with $\mathcal{Y}=\mathbb{R}^3 \times \mathbb{R}^3, \mathcal{Y}'=S^2 \times S^2$, $\phi(x, y)=(x / \|x\|,y / \|y\|) $. This method creates two orthonormal vectors.
    \begin{align}
        \pi: \mathcal{S}^2\times \mathcal{S}^2 &\rightarrow \SO3\\
       (\mathbf{u}, \vv) &\mapsto \text{concat}(\mathbf{w_1}, \mathbf{w_2}, \mathbf{w_3})\\
       &\text{Where:}\\
       &\mathbf{w_1} = \mathbf{u}\\
       &\mathbf{w_2}' = \vv - \langle \mathbf{u},\vv \rangle \mathbf{u}\\
       &\mathbf{w_2} = \frac{\mathbf{w_2}'}{\|\mathbf{w_2}'\|}\\
       &\mathbf{w_3} = \mathbf{w_1}\times \mathbf{w_2}
    \end{align}
 Notice that there exists a continuous and injective map $i:SO(3)\to S^2\times S^2$. It simply consists of taking the first two rows of the matrix representation of the $SO(3)$ element (The third row is the vector product between the first two, so it can always be recovered). Moreover we have that $\pi \circ i = \text{Id}_{SO(3)}$
\end{enumerate}

\section{Continuity Metric} \label{appendix:continuity}
Consider a map $f:X\to Y$ where $X$, $Y$ are metric spaces with metrics $d_X$ and $d_Y$ respectively.
In order to compute the proposed continuity metric we take a {\it continuous } path $(x_i)_{i\in[N]}$, defined as $N$ pairwise close points, and compute the relative distances
\begin{equation}
 L_i = \frac{d_Y(f(x_{i+1}), f(x_{i}))}{d_X(x_{i+1}, x_{i})}
\end{equation}
From this we further compute the quantities
\begin{equation}
 M:= \max_i L_i 
 \quad\text{and}\quad
 P_\alpha := \alpha\text{-th percentile of } \{L_i: i \in [N-1]\}.
\end{equation}
By comparing these two values, we want to discover whether there is at least one outlier in the set of $L_i$.
Such outliers corresponds to a transition with a {\it big} jump, signalling a discontinuity point.
We define a path to be discontinuous if $M > \gamma  P_{\alpha}$. 

In order to capture stochastic effects we repeat the above procedure with several paths.
The final score is the fraction of discontinuous paths.  
In the practical implementation we chose $1000$ paths, using $\gamma = 10$ and $\alpha = 90$ ($90$th percentile).

\section{Group Action}\label{appendix:group-action}
For each degree $l \in \mathbb{Z}_{\ge0}$, the Wigner-D-matrix can be expressed in a real basis as $D^l : \SO3 \rightarrow \mathbb{R}^{(2l+1) \times (2l+1)}$. We choose the $n_l$ copies of each degree $l$ and stack the matrices in block-diagonal form to create our representation, which amounts to taking the direct sum of the representations.

The Wigner-D-matrices represent rotations of the Fourier modes of a signal on the sphere, which provides an interpretation for using the group action. We consider a real signal one the sphere $f : S^2 \rightarrow \mathbb{R}$. It has a generalized Fourier transformation \citep{chirikjian2000engineering}:

$$ f(s) = \sum_{l=0}^\infty (2l+1) \sum_{m,n=-l}^l \hat{f}^l_m D^l_{m0}(\alpha_s, \beta_s, 0) $$

where $\hat{f}$ are the Fourier components and $D$ is the Wigner-D-matrix. We use identity $D^l_{m0}(\alpha, \beta, 0)=Y^l_m(\alpha, \beta)$, where $\alpha,\beta$ are the first two Euler angles, to write the spherical harmonics that are the basis functions of the Fourier modes as Wigner-D-matrices. Then for a rotation $g \in \SO3$, using the homomorphism property:

\begin{align*}
f(g(s)) &= \sum_{l=0}^\infty (2l+1) \sum_{m=-l}^l \hat{f}^l_m \sum_{r=-l}^l D^l_{mr}(g)D^l_{r0}(\alpha_s, \beta_s, 0) \\
 &= \sum_{l=0}^\infty (2l+1) \sum_{m=-l}^l \left( \sum_{r=-l}^l D^l_{mr}(g) \hat{f}^l_m \right) D^l_{r0}(\alpha_s, \beta_s, 0)
\end{align*}
where $g(s)$ corresponds to rotating a point on the sphere.

We see that our method of using representations in the decoder corresponds to having the content latent code represent the Fourier coefficients of a virtual signal on the sphere.

\section{Regularizers}\label{appendix:regularizers}

Even when an appropriate mean parametrization is selected and proper behaviour of the decoder is encouraged by the group action decoder, the network can still learn a discontinuous latent space. To encourage it to learn the data manifold correctly, we employ two additional loss terms that act as regularizers. An ablative analysis of the effectiveness of these regularizers is shown in Table \ref{tab:regularizers}.

\begin{table}[!tp]
\centering
\begin{tabular}{lccccc}
\toprule
Regularizer & NLL & ELBO & recon. . \\
\midrule
Neither & 235.24 & 246.2 & 231.2\\
Equivariance & 75.62 & 93.76 &  78.76\\
Continuity & 87.36 & 125.6 & 110.6\\
Both & \textbf{45.18} & \textbf{60.62} & \textbf{45.62}\\
\bottomrule
\end{tabular}
\caption{Ablative analysis of the regularizers. The has model uses $\SO3$ latent space, the group action decoder and the S2S2 mean.}
\label{tab:regularizers}
\end{table}

\subsection{Equivariance regularizer}
If the 3D object on which $\SO3$ acts is centered in the frame, then a $S^1$ subgroup $H$ of $\SO3$ exist such that its action corresponds to the rotations whose axis is orthogonal to the camera frame. For any $R \in \SO3$, angle $\theta$, decoder $g : \SO3 \rightarrow \mathbb{R}^N$, action of subgroup $H$ on the latent space, $\psi_\theta: \SO3 \to \SO3$ and action on the pixels through planar rotations $\phi_\theta:  \mathbb{R}^N \to \mathbb{R}^N$, we have equivariance relationship:
\begin{equation}
    g(\psi_\theta(R)) = \phi_\theta(g(R)) \label{eq:equivariance}
\end{equation}

This equivariance is shown in Figure \ref{fig:equivariance}. The relationship is exact if the object is centered, $\SO3$ acts on all pixels and if the camera is orthographic (located infinitely far away from the subject). If the object is off center, the pixel rotation can be performed around a learned center point. If the images have a rotation-invariant background, a learned mask can be applied. If the camera is not orthographic, the equivariance relationship is not exact, but approximate. The decoder is regularized by enforcing Equation \eqref{eq:equivariance} through  a mean squared error loss on the pixels for uniformly sampled $R \in \SO3$ and $\theta \in S^1$. We choose $\psi_\theta$ to correspond to rotation around the $z$-axis.

This regularizer helps align all rotations in each $S^1$ orbit, but does not help in correctly aligning the orbits among each other. Thus we reduce the problem from aligning $\SO3$ to aligning $\SO3 / S^1 \cong S^2$, since the cosets of $\SO3$ after the orbit are identified, are homeomorphic to the sphere.

\begin{figure}
    \centering
    \includegraphics[width=\textwidth]{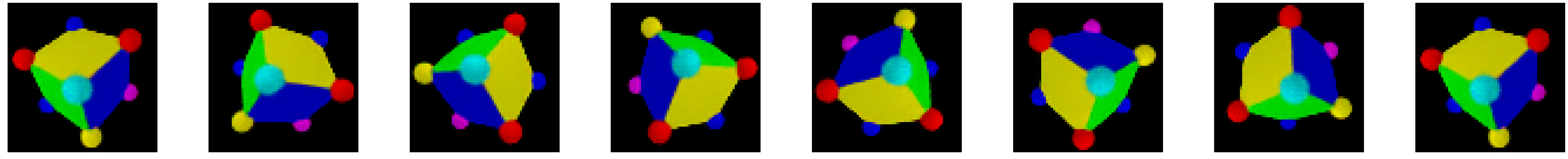}
    \caption{Decodings of an orbit of the equivariance subgroup $S^1$ whose axis is orthogonal to the camera frame. This orbit corresponds to planar rotations of the pixels.}
    \label{fig:equivariance}
\end{figure}

\subsection{Continuity regularizer}
If the learner is provided with pairs images that are nearby with respect to the manifold metric, the encoder can be regularized by penalizing differences in the encodings of the two inputs. This is done by penalizing the mean squared error of the Frobenius norms of the two encoded rotation matrices, which is a proper metric on the $\SO3$ manifold.

This simplifies the problem from unsupervised learning on i.i.d. samples to learning a VAE on two frame samples from random trajectories of data lying on the $\SO3$ manifold.

\end{document}